%% file: TSPCoopRev2.tex
\documentclass[journal, 10pt]{IEEEtran}
\usepackage{amsfonts}%
\usepackage{amssymb}%
\usepackage{setspace}
\usepackage[dvips]{graphicx}%
\usepackage{amsmath}%
\usepackage[usenames]{color}
\usepackage{algorithmic}
\usepackage{cite}
\usepackage{array}
\newtheorem{theorem}{Theorem}
\newtheorem{corollary}{Corollary}
\newtheorem{lemma}{Lemma}
\newtheorem{assumption}{Assumption}
\newtheorem{definition}{Definition}
\newtheorem{example}{Example}

\newtheorem{remark}{Remark}
\allowdisplaybreaks 

\DeclareMathOperator*{\argmin}{arg\,min}
\DeclareMathOperator*{\argmax}{arg\,max}

\ifodd 1
\newcommand{\ntsp}[1]{{\color{black}#1}}
\else
\newcommand{\ntsp}[1]{#1}
\fi

\ifodd 0
\newcommand{\tsp}[1]{{\color{red}#1}}
\else
\newcommand{\tsp}[1]{#1}
\fi

\newcommand{\comment}[1]{}

\ifodd 0
\newcommand{\remove}[1]{}
\newcommand{\add}[1]{#1}
\else
\newcommand{\remove}[1]{#1}
\newcommand{\add}[1]{}
\fi

\ifodd 1
\newcommand{\rmv}[1]{}
\else
\newcommand{\rmv}[1]{{\color{red}#1}}
\fi

\ifodd 1
\newcommand{\rmvt}[1]{}
\else
\newcommand{\rmvt}[1]{{\color{red}#1}}
\fi

\ifodd 0
\newcommand{\newt}[1]{{\color{blue}#1}}
\else
\newcommand{\newt}[1]{#1}
\fi

\ifodd 0
\newcommand{\newton}[1]{{\color{blue}#1}}
\else
\newcommand{\newton}[1]{#1}
\fi

\ifodd 0
\newcommand{\marrew}[1]{{\color{blue}#1}}
\else
\newcommand{\marrew}[1]{#1}
\fi

\ifodd 0
\newcommand{\rev}[1]{{\color{blue}#1}} 
\newcommand{\cem}[1]{{\color{magenta}#1}} 
\newcommand{\com}[1]{\textbf{\color{red}(COMMENT: #1)}} 
\newcommand{\clar}[1]{\textbf{\color{green}(NEED CLARIFICATION: #1)}}

\else
\newcommand{\rev}[1]{#1}
\newcommand{\cem}[1]{#1}

\newcommand{\com}[1]{}
\newcommand{\clar}[1]{}
\fi

\begin{document}
\title{Distributed Online Learning via Cooperative Contextual Bandits}
\author{\IEEEauthorblockN{Cem Tekin*,~\IEEEmembership{Member,~IEEE}, Mihaela van der Schaar, ~\IEEEmembership{Fellow,~IEEE}\\}
\thanks{C. Tekin and M. van der Schaar are with the Department of Electrical Engineering, UCLA. Email:cmtkn@ucla.edu, mihaela@ee.ucla.edu.}
\thanks{A preliminary version of this work appeared in Allerton 2013. The work is partially supported by the grants NSF CNS 1016081 and  AFOSR DDDAS.}
}

\maketitle

\begin{abstract}
In this paper we propose a novel framework for decentralized, online learning by many learners.  
At each moment of time, an instance characterized by a certain context may arrive to each learner; 
based on the context, the learner can select one of its own actions (which gives a reward and provides information) or request assistance from another learner.  In the latter case, the requester pays a cost and receives the reward but the provider learns the information. 
In our framework, learners are modeled as cooperative contextual bandits.
Each learner seeks to maximize the expected reward from its arrivals, which involves trading off the reward received from its own actions, the information learned from its own actions, the reward received from the actions requested of others and the cost paid for these actions - taking into account what it has learned about the value of assistance from each other learner.
We develop distributed online learning algorithms and provide analytic bounds to compare the efficiency of these with algorithms with the complete knowledge (oracle) benchmark (in which the expected reward of every action in every context is known by every learner).  Our estimates show that regret - the loss incurred by the algorithm - is sublinear in time.  Our theoretical framework can be used in many practical applications including Big Data mining, event detection in surveillance sensor networks and distributed online recommendation systems.
\end{abstract}


\input{intro}
\input{related}
\input{probform}
\input{iid_scheme}

\input{cos_analysis}
\input{zooming_scheme}

\input{discussion}
\vspace{-0.1in}
\section{Conclusion} \label{sec:conc}
In this paper we proposed a novel framework for
decentralized, online learning by many learners.
We developed two novel online learning algorithms for this problem and proved sublinear regret results for our algorithms. We discussed some implementation issues such as complexity and the memory requirement under different \newton{instance and context arrivals}.
Our theoretical framework can be applied to many practical settings including distributed online learning in Big Data mining, recommendation systems and surveillance applications.
Cooperative contextual bandits opens a new research direction in online learning and raises many interesting questions:
What are the lower bounds on the regret? Is there a gap in the time order of the lower bound compared to centralized contextual bandits due to informational asymmetries? 
Can regret bounds be proved when cost of calling learner $j$ is controlled by learner $j$? In other words, what happens when a learner wants to maximize both the total reward from its own contexts and the total reward from the calls of other learners.

\rmv{An interesting research direction is to see the performance of CoS when combined with ensemble learning approaches. This will increase the communication and computation costs of the learners but the improvement in classification accuracy can be large compared to existing online ensemble learning methods such as SWA.}
\vspace{-0.1in}

\remove{
 \appendices
 \section{A bound on divergent series} \label{app:seriesbound}
 \input{app_seriesbound}
 \vspace{-0.1in}
 \section{Frequently used expressions} \label{app:notation}
 \input{notation}
}
\vspace{-0.1in}
\bibliographystyle{IEEE}
\bibliography{OSA}


\end{document}

%% file: intro.tex
\vspace{-0.1in}
\section{Introduction}\label{sec:intro}

In this paper we propose a novel framework for online learning by
multiple cooperative and decentralized learners. We assume that an
instance (a data unit), characterized by a context (side) information,
\newton{ arrives} at a learner (processor) which needs to process
it either by using one of its own processing functions or by requesting
another learner (processor) to process it. The learner's goal is to
learn online what is the best processing function which it should
use such that it maximizes its total expected reward for that instance.
\marrew{A data stream is an ordered sequence of instances that can
be read only once or a small number of times using limited computing
and storage capabilities. For example, in a stream mining application, an instance can be the data
unit extracted by a sensor or camera; in a wireless communication
application, an instance can be a packet that needs to be transmitted.}
The context can be anything that provides information about the rewards
to the learners. For example, in stream mining, the context
can be the type of the extracted instance; in wireless communications,
the context can be the channel Signal to Noise Ratio (SNR). The processing
functions in the stream mining application can be the various classification
functions, while in wireless communications they can be the transmission
strategies for sending the packet (Note that the selection of the
processing functions by the learners can be performed based on the
context and not necessarily the instance). The rewards in the stream
mining can be the accuracy associated with the selected classification
function, and in wireless communication they can be the resulting goodput
and expended energy associated with a selected transmission strategy.

\newton{To solve such distributed online learning problems, we define
a new class of multi-armed bandit solutions, which we refer to as
{\em cooperative contextual bandits.}} In the considered scenario,
there is a set of \newton{cooperative learners, each equipped with
a set of processing functions (arms\footnote{We use the terms action and arm interchangeably.}) which can be used
to process the instance}. \tsp{By definition, cooperative learners agree to follow the rules of a prescribed algorithm provided by a designer given that the prescriped algorithm meets the set of constraints imposed by the learners. For instance, these constraints can be privacy constraints, which limits the amount of information a learner knows about the arms of the other learners.} We assume a discrete time model $t=1,2,\ldots$,
where different \newton{instances and associated context information}
arrive to a learner.\footnote{Assuming synchronous agents/learners is common in the decentralized
multi-armed bandit literature \cite{liu2010distributed,tekin2012sequencing}.
Although our formulation is for synchronous learners, our results
directly apply to the asynchronous learners, where times of instance
and context arrivals can be different. A learner may not receive an
instance and context at every time slot $t$. Then, instead of the
final time $T$, our performance bounds for learner $i$ will depend
on the total number of arrivals to learner $i$ by time $T$.%
} Upon the arrival of an \newton{instance}, a learner needs to select
either one of its arms to process the instance or it can call another
learner which can select one of its own arms to process the instance and
incur a cost \newton{(e.g., delay cost, communication cost, processing
cost, money)}.
Based on the selected arm, the learner receives a random
reward, which is drawn from some unknown distribution that depends
on the context information characterizing the instance. The goal of
a learner is to maximize its total undiscounted reward up to any time
horizon $T$. A learner does not know the expected reward (as a function
of the context) of its own arms or of the other learners' arms. In
fact, we go one step further and assume that a learner does not know
anything about the set of arms available to other learners except
an upper bound on the number of their arms.
The learners are cooperative because they obtain mutual benefits from cooperation
- a learner's benefit from calling another learner may be an increased
reward as compared to the case when it uses solely its own arms; the
benefit of the learner asked to perform the processing by another
learner is that it can learn about the performance of its own arm
based on its reward for the calling learner.
This is especially beneficial when certain instances and associated contexts are less
frequent, or when gathering labels (observing the reward) is costly.

The problem defined in this paper is a generalization of the well-known
contextual bandit problem \cite{kleinberg2008multi,bubeck2011x,slivkins2009contextual,dudik2011efficient,langford2007epoch,chu2011contextual},
in which there is a single learner who has access to all the arms.
\newton{However, the considered distributed online learning problem
is significantly more challenging because a learner cannot observe
the arms of other learners and cannot directly estimate the expected
rewards of those arms. Moreover, the heterogeneous contexts arriving
at each learner lead to different learning rates for the various learners.}
We design distributed online learning algorithms whose long-term average
rewards converge to the best distributed solution which can \newton{be
obtained if we assumed complete knowledge of the expected arm rewards
of each learner for each context.} 

To rigorously quantify the learning performance, we define the regret
of an online learning algorithm for a learner as the difference between
the expected total reward of the best decentralized arm selection
scheme given complete knowledge about the expected arm rewards of
all learners and the expected total reward of the algorithm used by
the learner.
Simply, the regret of a learner is the loss incurred due to the unknown system dynamics compared to the complete knowledge benchmark. 
We prove a sublinear upper bound on the regret, 
which implies that the average reward converges to the optimal average
reward. The upper bound on regret gives a lower bound on the convergence
rate to the optimal average reward. 
We show that when the contexts arriving to a learner are uniformly
distributed over the context space, 
the regret depends on the dimension of the context space, while when
the contexts arriving to the same learner are concentrated in a small
region of the context space, 
the regret is independent of the dimension of the context space. 

The proposed framework can be used in numerous applications including
the ones given below. 


\begin{example}
\ntsp{
Consider a distributed recommender system in which there is a group of agents (learners) that are connected together via a fixed network, each of whom experiences inflows of users to its page. Each time a user arrives, an agent chooses from among a set of items (arms) to offer to that user, and the user will either reject or accept each item. When choosing among the items to offer, the agent is uncertain about the user's acceptance probability of each item, but the agent is able to observe specific background information about the user (context), such as the user's gender, location, age, etc. Users with different backgrounds will have different probabilities of accepting each item, and so the agent must learn this probability over time by making different offers. 
In order to promote cooperation within this network, we let each agent also recommend items of other agents to its users in addition to its own items. Hence, if the agent learns that a user with a particular context is unlikely to accept any of the agent's items, it can recommend to the user items of another agent that the user might be interested in. The agent can get a commission from the other agent if it sells the item of the other agent. This provides the necessary incentive to cooperate. 
However, since agents are decentralized, they do not directly share the information that they learn over time about user preferences for their own items. Hence the agents must learn about other agent's acceptance probabilities through their own trial and error.}
\end{example}

\begin{example}
Consider a network security scenario in which autonomous systems (ASs)
collaborate with each other to detect cyber-attacks \cite{cem2013contextdata}.
Each AS has a set of security solutions which it can use to detect
attacks. The contexts are the characteristics of the data traffic
in each AS. These contexts can provide valuable information about
the occurrence of cyber-attacks. Since the nature of the attacks are
dynamic, non-stochastic and context dependent, the efficiency of the
various security solutions are dynamically varying, context dependent
and unknown a-priori. Based on the extracted contexts (e.g. key properties
of its traffic, the originator of the traffic etc.), an AS $i$ may
route its incoming data stream (or only the context information) to
another AS $j$, and if AS $j$ detects a malicious activity based
on its own security solutions, it warns AS $i$. Due to the privacy
or security concerns, AS $i$ may not know what security applications
AS $j$ is running. This problem can be modeled as a cooperative contextual
bandit problem in which the various ASs cooperate with each other
to learn online which actions they should take or which other ASs
they should request to take actions in order to accurately detect
attacks (e.g. minimize the mis-detection probability of cyber-attacks). 
\end{example}
\comment{ 
\begin{example}
In cognitive radio channel access \cite{liu2010distributed,anandkumar2011distributed,tekin2012online},
decentralized secondary users learn the optimal allocation of channels
to maximize the sum of the utilities of the secondary users, where
coexistence in a channel is not possible. If secondary users are allowed
to transmit in the same channel with the primary users without providing
harmful interference, estimates of the channel qualities and primary
user interference temperature constraint can be used as context information
to maximize the utility of the secondary users. Therefore, this problem
can be modeled as a cooperative contextual bandit problem. 
\end{example}
}


The remainder of the paper is organized as follows. In Section \ref{sec:related}
we describe the related work and highlight the differences from our
work. \marrew{In Section \ref{sec:probform} we describe the choices
of learners, rewards, complete knowledge benchmark, and define the
regret of a learning algorithm.} A cooperative contextual learning
algorithm that uses a non-adaptive partition of the context space
is proposed and a sublinear bound on its regret is derived in Section
\ref{sec:iid}. Another learning algorithm that adaptively partitions
the context space of each learner is proposed in Section \ref{sec:zooming},
and its regret is bounded for different types of context arrivals.
In Section \ref{sec:discuss} we discuss the necessity of {\em training
phase} which is a property of both algorithms and compare them.
Finally, the concluding remarks are given in Section \ref{sec:conc}.

%% file: related.tex
\vspace{-0.15in}
\section{Related Work} \label{sec:related}

Contextual bandits have been studied before in \cite{slivkins2009contextual, dudik2011efficient, langford2007epoch, chu2011contextual} in a single agent setting, where the agent sequentially chooses from a set of arms with unknown rewards, and the rewards depend on the context information provided to the agent at each time slot. 
The goal of the agent is to maximize its reward by balancing exploration of arms with uncertain rewards and exploitation of the arm with the highest estimated reward.
The algorithms proposed in these works \tsp{are} shown to achieve sublinear in time regret with respect to the complete knowledge benchmark, and the sublinear regret bounds are proved to match with lower bounds on the regret up to logarithmic factors. 
In all the prior work, the context space is assumed to be large and a known similarity metric over the contexts is exploited by the algorithms to estimate arm rewards together for groups of similar contexts. Groups of contexts are created by partitioning the context space. 
For example, \cite{langford2007epoch} proposed an epoch-based uniform partition of the context space, while \cite{slivkins2009contextual} proposed a non-uniform adaptive partition. 
In \cite{li2010contextual}, contextual bandit methods are developed for personalized news articles recommendation and a variant of the UCB algorithm \cite{auer} is designed for linear payoffs.
In \cite{crammer2011multiclass}, contextual bandit methods are developed for data mining and a perceptron based algorithm that achieves sublinear regret when the instances are chosen by an adversary is proposed. 
\newton{To the best of our knowledge,
our work is the first to provide rigorous solutions for online learning by multiple cooperative learners when context information is present and propose a novel framework for cooperative contextual bandits to solve this problem.}

Another line of work \cite{kleinberg2008multi, bubeck2011x} considers a single agent with a large set of arms (often uncountable).
Given a similarity structure on the arm space, they propose online learning algorithms that adaptively partition the arm space to get sublinear regret bounds.
The algorithms we design in this paper also exploits the similarity information, but in the context space rather than the action space, to create a partition and learn through the partition. However, distributed problem formulation, creation of the partitions and how learning is performed is very different from related prior work \cite{slivkins2009contextual, dudik2011efficient, langford2007epoch, chu2011contextual, kleinberg2008multi, bubeck2011x}.

Previously, distributed multi-user learning is only considered for multi-armed bandits with finite number of arms and no context. In \cite{anandkumar, liu2010distributed} distributed online learning algorithms that converge to the optimal allocation with logarithmic regret are proposed for the i.i.d. arm reward model, given that the optimal allocation is an orthogonal allocation in which each user selects a different arm. Considering a similar model but with Markov arm rewards, logarithmic regret algorithms are proposed in \cite{tekin2012online, 6362216}, where the regret is with respect to the best static policy which is not generally optimal for Markov rewards. 
This is generalized in \cite{tekin2012sequencing} to dynamic resource sharing problems and logarithmic regret results are also proved for this case. 

\tsp{
A multi-armed bandit approach is proposed in \cite{stranders2012dcops} to solve {\em decentralized constraint optimization problems} (DCOPs) with unknown and stochastic utility functions. The goal in this work is to maximize the total cumulative reward, where the cumulative reward is given as a sum of {\em local} utility functions whose values are controlled by variable assignments made (actions taken) by a subset of agents. The authors propose a message passing algorithm to efficiently compute a global upper confidence bound on the joint variable assignment, which leads to logarithmic in time regret. In contrast, in our formulation we consider a problem in which rewards are driven by contexts, and the agents do not know the set of actions of the other agents. 
In \cite{gai2012combinatorial} a combinatorial multi-armed bandit problem is proposed in which the reward is a linear combination of a set of coefficients of a multi-dimensional action vector and an instance vector generated by an unknown i.i.d. process. They propose an upper confidence bound algorithm that computes a global confidence bound for the action vector which is the sum of the upper confidence bounds computed separately for each dimension. Under the proposed i.i.d. model, this algorithm achieves regret that grows logarithmically in time and polynomially in the dimension of the vector.
}

We provide a detailed comparison between our work and related work in multi-armed bandit learning in Table \ref{tab:comparison2}.
Our cooperative contextual learning framework can be seen as an important extension of the centralized contextual bandit framework \cite{kleinberg2008multi, bubeck2011x, slivkins2009contextual, dudik2011efficient, langford2007epoch, chu2011contextual}.
The main differences are:
(i) {\em training} phase which is required due to the informational asymmetries between learners, 
(ii) separation of exploration and exploitation over time instead of using an index for each arm to balance them, resulting in three-phase learning algorithms with {\em training}, {\em exploration} and {\em exploitation} phases,
(iii) coordinated context space partitioning in order to balance the differences in reward estimation due to heterogeneous context arrivals to the learners.
\marrew{Although we consider a three-phase learning structure, our learning framework can work together with  index-based policies such as the ones proposed in \cite{slivkins2009contextual}, by restricting the index updates to time slots that are not in the training phase. 
Our three-phase learning structure separates exploration and exploitation into distinct time slots, while they take place concurrently for an index-based policy. 
We will discuss the differences between these methods in Section \ref{sec:discuss}.
We will also show in Section \ref{sec:discuss} that the training phase is necessary for the learners to form correct estimates about each other's rewards in cooperative contextual bandits.}

Different from our work, distributed learning is also considered in online convex optimization setting \cite{ram2010distributed, yan2013distributed, raginsky2011decentralized}.
In all of these works local learners choose their actions (parameter vectors) to minimize the global total loss by exchanging messages with their neighbors and performing subgradient descent. 
In contrast to these works in which learners share information about their actions, the learners in our model does not share any information about their own actions. The information shared in our model is the context information of the calling learner and the reward generated by the arm of the called learner. However, this information is not shared at every time slot, and the rate of information sharing between learners who cannot help each other to gain higher rewards goes to zero asymptotically.

\ntsp{
In addition to the aforementioned prior work, in our recent work \cite{tekin2014recommender} we consider online learning in a decentralized social recommender system. In this related work, we address the challenges of decentralization, cooperation, incentives and privacy that arises in a network of recommender systems. We model the item recommendation strategy of a learner as a combinatorial learning problem, and prove that learning is much faster when the purchase probabilities of the items are independent of each other. In contrast, in this work we propose the general theoretical model of cooperative contextual bandits which can be applied in a variety of decentralized online learning settings including wireless sensor surveillance networks, cognitive radio networks, network security applications, recommender systems, etc. We show how context space partition can be adapted based on the context arrival process and prove the necessity of the training phase. 
}

\begin{table}[t]
\centering
{\fontsize{9}{8}\selectfont
\setlength{\tabcolsep}{.25em}
\vspace{-0.2in}
\begin{tabular}{|l|c|c|c|c|}
\hline
&\cite{slivkins2009contextual, dudik2011efficient, langford2007epoch, chu2011contextual} &  \cite{hliu1, anandkumar, tekin2012sequencing}  & This work  \\
\hline
Multi-user & no & yes  & yes \\
\hline
Cooperative & N/A & yes  & yes  \\
\hline
Contextual & yes & no  & yes  \\
\hline
Context arrival  & arbitrary & N/A  & arbitrary  \\
process& & &  \\
\hline
synchronous (syn)/& N/A & syn & both  \\
asynchronous (asn)& & & \\
\hline
Regret & sublinear & logarithmic  & sublinear  \\
\hline
\end{tabular}
}
\caption{Comparison with related work in multi-armed bandits}
\vspace{-0.35in}
\label{tab:comparison2}
\end{table}

%% file: probform.tex
\add{\vspace{-0.2in}}
\vspace{-0.1in}
\section{Problem Formulation}\label{sec:probform}

The system model is shown in Fig. \ref{fig:system}. There are $M$ learners which are indexed by the set ${\cal M} = \{1,2,\ldots,M\}$.
Let ${\cal M}_{-i} := {\cal M} - \{i\}$ be the set of learners learner $i$ can choose from to receive a reward.
Let ${\cal F}_i$ denote the set of {\em arms} of learner $i$.
Let ${\cal F} := \cup_{j \in {\cal M}} {\cal F}_j$ denote the {\em set of all arms}.
Let ${\cal K}_i := {\cal F}_i \cup {\cal M}_{-i}$.
We call ${\cal K}_i$ the set of \marrew{{\em choices}} for learner $i$.
\newt{We use index $k$ to denote any choice in ${\cal K}_i$, $f$ to denote arms of the learners, $j$ to denote other learners in ${\cal M}_{-i}$.}
\tsp{Let $M_i :=|{\cal M}_{-i}|$, $F_i :=|{\cal F}_i|$ and $K_i := |{\cal K}_i|$, where $|\cdot|$ is the cardinality operator. A summary of notations is provided in Appendix \ref{app:notation}.}

\tsp{The learners operate under the following privacy constraint: A learner's set of arms is its private information. This is important when the learners want to cooperate to maximize their rewards, but do not want to reveal their technology/methods. For instance in stream mining, a learner may not want to reveal the types of classifiers it uses to make predictions, or in network security a learner may not want to reveal how many nodes it controls in the network and what types of security protocols it uses. However, each learner knows an upper bound on the number of arms the other learners have. Since the learners are cooperative, they can follow the rules of any learning algorithm as long as the proposed learning algorithm satisfies the privacy constraint. In this paper, we design such a learning algorithm and show that it is optimal in terms of average reward.}

These learners work in a discrete time setting $t=1,2,\ldots,T$, where the following events happen sequentially, in each time slot:
(i) an instance with context $x_i(t)$ arrives to each learner $i \in {\cal M}$;
(ii) based on $x_i(t)$, learner $i$ either chooses one of its arms $f \in {\cal F}_i$ or calls another learner and sends $x_i(t)$;\footnote{\newton{An alternative formulation is that learner $i$ selects multiple choices from ${\cal K}_i$ at each time slot, and receives sum of the rewards of the selected choices. All of the ideas/results in this paper can be extended to this case as well.}}
(iii) for each learner who called learner $i$ at time $t$, learner $i$ chooses one of its arms $f \in {\cal F}_i$;
(iv) learner $i$ observes the rewards of all the arms $f \in {\cal F}_i$ it had chosen both for its own contexts and for other learners;
(v) learner $i$ either obtains directly the reward of its own arm it had chosen, or a reward that is passed from the learner that it had called for its own context.\footnote{Although in our problem description the learners are synchronized, our model also works for the case where instance/context arrives asynchronously to each learner. We discuss more about this in \cite{cem2013contextdata}.}
%
%

%
%
The contexts $x_i(t)$ come from a bounded $D$ dimensional space ${\cal X}$, which is taken to be $[0,1]^D$ without loss of generality.
When selected, an arm $f \in {\cal F}$ generates a random reward sampled from an unknown, context dependent distribution $G_f(x)$ with support in $[0,1]$.\footnote{Our results can be generalized to rewards with bounded support $[b_1, b_2]$ for $-\infty < b_1 < b_2 < \infty$. This will only scale our performance bounds by a constant factor.}
The expected reward of arm $f \in {\cal F}$ for context $x \in {\cal X}$ is denoted by $\pi_{\tsp{f}}(x)$.
Learner $i$ incurs a known deterministic and fixed cost $d^i_k$ for selecting choice $k \in {\cal K}_i$.\footnote{Alternatively, we can assume that the costs are random variables with bounded support whose distribution is unknown.
In this case, the learners will not learn the reward but they will learn reward minus cost which is essentially the same thing. However, our performance bounds will be scaled by a constant factor.}
For example for $k \in {\cal F}_i$, $d^i_{k}$ can represent the cost of activating arm $k$, while for $k \in {\cal M}_{-i}$, $d^i_{k}$ can represent the cost of communicating with learner $k$ and/or the payment made to learner $k$.
Although in our system model we assume that each learner $i$ can directly call another learner $j$, our model can be generalized to learners over a network where calling learners that are away from learner $i$ has a higher cost for learner $i$.
%
Learner $i$ knows the set of other learners ${\cal M}_{-i}$ and costs of calling them, i.e., $d^i_{j}, j \in {\cal M}_{-i}$, but does not know the set of arms ${\cal F}_{j}$, $j \in {\cal M}_{-i}$, but only knows an upper bound on the number of arms that each learner has, i.e., $F_{\max}$ on \tsp{$F_j$}, $j \in {\cal M}_{-i}$.
%
%
%
Since the costs are bounded, without loss of generality we assume that costs are normalized, i.e., $d^i_k \in [0,1]$ for $k \in {\cal K}_i$, $i \in {\cal M}$.
\marrew{The {\em net reward} of learner $i$ from a choice is equal to the obtained reward minus cost of selecting the choice.
The net reward of a learner is always in $[-1,1]$.
}
%
%

The learners are cooperative which implies that when called by learner $i$, learner $j$ will choose one of its own arms which it believes to yield the highest expected reward given the context of learner $i$.
%
%

%


\begin{figure}
\begin{center}
\includegraphics[width=0.9\columnwidth]{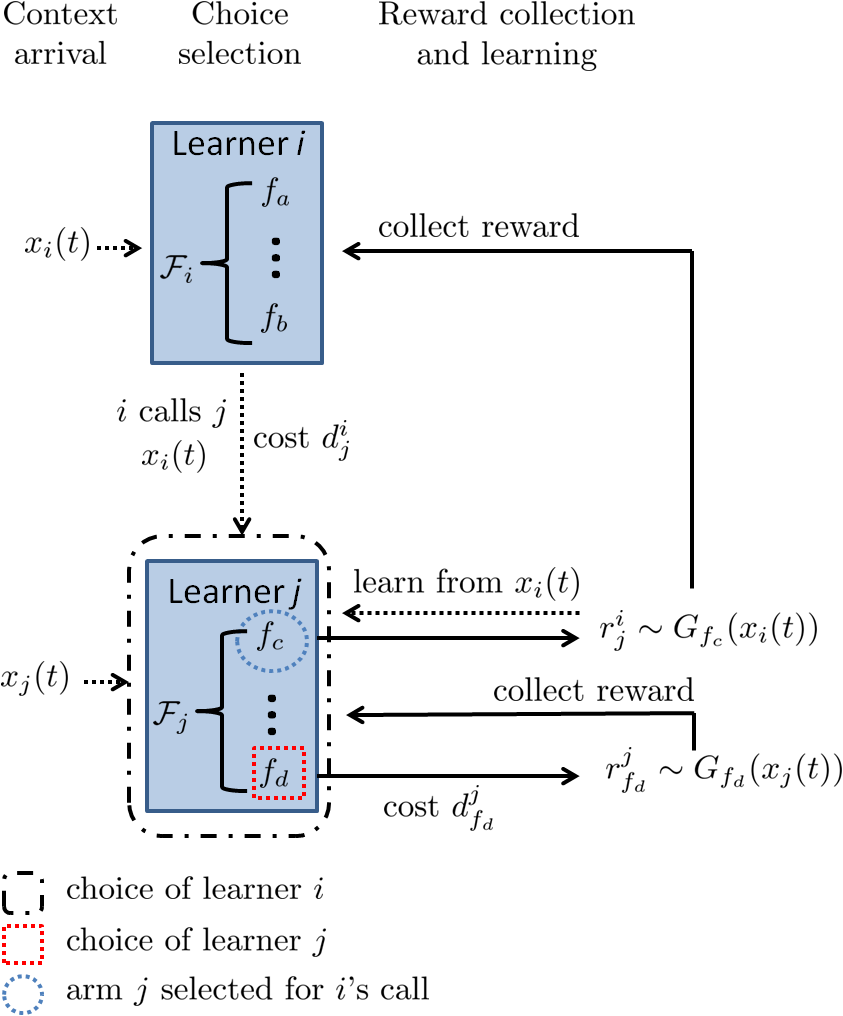}
\vspace{-0.1in}
\caption{System model from the viewpoint of learners $i$ and $j$. Here $i$ exploits $j$ to obtain a high reward while helping $j$ to learn about the reward of its own arm.} 
\vspace{-0.2in}
\label{fig:system}
\end{center}
\end{figure}

The expected reward of an arm is similar for similar contexts, which is formalized in terms of a H\"{o}lder condition given in the following assumption.
\begin{assumption} \label{ass:lipschitz2}
There exists $L>0$, $\alpha>0$ such that for all $f \in {\cal F}$ and for all $x,x' \in {\cal X}$, we have
$|\pi_{f}(x) - \pi_{f}(x')| \leq L ||x-x'||^\alpha$,
where $||\cdot||$ denotes the Euclidian norm in $\mathbb{R}^D$.
\end{assumption}
We assume that $\alpha$ is known by the learners. In the contextual bandit literature this is referred to as {\em similarity information} \cite{slivkins2009contextual}, \cite{ortner2010exploiting}.
Different from prior works on contextual bandit, we do not require $L$ to be known by the learners. However, $L$ will appear in our performance bounds.

The goal of learner $i$ is to maximize its total expected reward. In order to do this, it needs to learn the rewards from its choices.
Thus, learner $i$ should concurrently explore the choices in ${\cal K}_i$ to learn their expected rewards, and exploit the best believed choice for its contexts which maximizes the reward minus cost. 
%
%
%
In the next subsection we formally define the complete knowledge benchmark. Then, we define the regret which is the performance loss due to uncertainty about arm rewards.

\vspace{-0.1in}
\subsection{Optimal Arm Selection Policy with Complete Information} \label{sec:centralized}

We define learner $j$'s expected reward for context $x$ as $\pi_j(x) :=  \pi_{f^*_j(x)}(x)$,
where $f^*_j(x) := \argmax_{ f \in {\cal F}_j  } \pi_f(x)$.
 This is the maximum expected reward learner $j$ can provide when called by a learner with context $x$.
For learner $i$, $\mu^i_k(x) := \pi_k(x) - d^i_k$ denotes the net reward of choice $k \in {\cal K}_i$ for context $x$.
Our benchmark when evaluating the performance of the learning algorithms is the optimal solution which selects the choice with the highest expected net reward for learner $i$ for its context $x$. 
This is given by
\begin{align}
k_i^*(x) := \argmax_{k \in {\cal K}_i} \mu^i_k(x)  ~~ \forall x \in {\cal X}. \label{eqn:opt2}
\end{align}
Since knowing $\mu^i_j(x)$ requires knowing $\pi_f(x)$ for $f \in {\cal F}_j$, knowing the optimal solution means that learner $i$ knows the arm in ${\cal F}$ that yields the highest expected reward for each $x \in {\cal X}$.
\rmv{
Note that the problem remains hard even if we would put additional structure on the optimal classification scheme. For instance, we could assume that the 
%
optimal classification scheme for learner $i$ partitions ${\cal X}$ into  $|\cup_{i \in {\cal M}} {\cal F}_i|$ sets in each of which a single classification function is optimal.
%
Assume for an instance that the data arrival process to learner $i$ is i.i.d. with density $q_i$ and learner $i$ has access to all classification functions $\cup_{i \in {\cal M}} {\cal F}_i$. If learner $i$ knows $q_i$ and the classification accuracies of all classification functions in $\cup_{i \in {\cal M}} {\cal F}_i$, then 
under the above assumption, learner $i$ could compute the optimal classification regions
\begin{align*}
{\cal R}^* = \{R^*_k\}_{(k \in \cup_{i \in {\cal M}} {\cal F}_i)},
\end{align*}
of ${\cal X}$, by solving
\begin{align}
\textrm{{\bf (P1)  }} {\cal R}^* = \argmin_{{\cal R} \in \Theta} E \left[ \left| Y - \sum_{k \in \cup_{i \in {\cal M}} {\cal F}_i} k(X) I(X \in R_l) \right| + \sum_{k \in {\cal F}_i} d_k I(X \in R_k) + \sum_{k \in {\cal M}_{-i}}  d_k I \left(X \in \cup_{j \in {\cal F}_k} R_j \right)   \right], \label{eqn:opt}
\end{align}
where $\Theta$ is the set of $|\cup_{i \in {\cal M}} {\cal F}_i|$-set partitions of ${\cal X}$, the expectation is taken with respect to the distribution $q_i$ and $Y$ is the random variable denoting the true label. Here $I(X \in R_l)$ denotes the event that data received by the learner belongs to the $l$-th set of the partition ${\cal R}$ of ${\cal X}$.
The complexity of finding the optimal classification regions increases exponentially with $|\cup_{i \in {\cal M}} {\cal F}_i|$.
Importantly, note that the learning problem we are trying to solve is even harder than this because the learners are distributed and thus, each learner cannot directly access to all classification functions (learner $i$ only knows set ${\cal F}_i$ and ${\cal M}$, but it does not know any ${\cal F}_j, j \in {\cal M}_{-i}$), and the distributions $q_i$, $i \in {\cal M}$ are unknown (they need not to be i.i.d. or Markovian). Therefore, we use online learning techniques that do not rely on solving the optimization problem in (\ref{eqn:opt}) nor on an estimated version of it. 

\cem{Note that we discussed the above assumption only to illustrate that the optimal solution is still computationally
hard even when we put extra structure on the problem. Our results  
only require Assumption \ref{ass:lipschitz2} to hold.}
}
 
%
%
\vspace{-0.1in}
\subsection{The Regret of Learning}

Let $a_i(t)$ be the choice selected by learner $i$ at time $t$. 
Since learner $i$ has no a priori information, this choice is only based on the past history of selections and reward observations of learner $i$.
The rule that maps the history of learner $i$ to its choices is called the learning algorithm of learner $i$.
Let $\boldsymbol{a}(t) := (a_1(t), \ldots, a_M(t))$ be the choice vector at time $t$. 
We let $b_{i,j}(t)$ denote the arm selected by learner $i$ when it is called by learner $j$ at time $t$.
If $j$ does not call $i$ at time $t$, then $b_{i,j}(t) = \emptyset$.
Let $\boldsymbol{b}_i(t) = \{ b_{i,j}(t) \}_{j \in {\cal M}_{-i}}$ and $\boldsymbol{b}(t) = \{ \boldsymbol{b}_{i}(t) \}_{i \in {\cal M}}$.
The regret of learner $i$ with respect to the complete knowledge benchmark $k_i^*(x_i(t))$ given in (\ref{eqn:opt2}) is given by
\add{\vspace{-0.05in}}
\begin{align*}
R_i(T) &:= \sum_{t=1}^T \left(\pi_{k_i^*(x_i(t))}(x_i(t)) - d^i_{k_i^*(x_i(t))} \right)\\
&- \mathrm{E} \left[ \sum_{t=1}^T  r^i_{a_i(t) }(x_i(t),t)  - d^i_{a_i(t) } \right] 
\end{align*}
where $r^i_{a_i(t)}(x_i(t),t)$ denotes the random reward of choice $a_i(t) \in {\cal K}_i$ for context $x$ at time $t$ for learner $i$, and the expectation is taken with respect to the selections made by the distributed algorithm of the learners and the statistics of the rewards. For example, when $a_i(t) = j$ and $b_{j,i}(t) = f \in {\cal F}_j$, this random reward is sampled from the distribution of arm $f$.

Regret gives the convergence rate of the total expected reward of the learning algorithm to the value of the optimal solution given in (\ref{eqn:opt2}).
Any algorithm whose regret is sublinear, i.e., $R(T) = O(T^\gamma)$ such that $\gamma<1$, will converge to the optimal solution in terms of the average reward.
In the subsequent sections we will propose two different distributed learning algorithms with sublinear regret.

%% file: iid_scheme.tex
\vspace{-0.1in}
\section{A distributed uniform context partitioning algorithm} \label{sec:iid}

\marrew{The algorithm we consider in this section forms at the beginning a uniform partition of the context space for each learner. Each learner estimates its choice rewards based on the past history of arrivals to each set in the partition independently from the other sets in the partition.}
This distributed learning algorithm is called {\em Contextual Learning with Uniform Partition} (CLUP) and its pseudocode is given in Fig. \ref{fig:CLUP}, Fig. \ref{fig:CLUPmax} and Fig. \ref{fig:CLUPcoop}.
\marrew{For learner $i$, CLUP is composed of two parts. The first part is the {\em maximization part} (see Fig. \ref{fig:CLUPmax}), which is used by learner $i$ to maximize its reward from its own contexts. The second part is the {\em cooperation part} (see Fig. \ref{fig:CLUPcoop}), which is used by learner $i$ to help other learners maximize their rewards for their own contexts.}

Let $m_T$ be the {\em slicing parameter} of CLUP that determines the \marrew{number of sets in the} partition of the context space ${\cal X}$. 
When $m_T$ is small, the number of sets in the partition is small, hence the number of contexts from the past observations which can be used to form reward estimates in each set is large.
However, when $m_T$ is small, the size of each set is large, hence the variation of the expected \marrew{choice} rewards over each set is high.
First, we will analyze the regret of CLUP for a fixed $m_T$ and then optimize over it to balance the aforementioned tradeoff. 
CLUP forms a partition of $[0,1]^D$ consisting of $(m_T)^D$ sets where each set is a $D$-dimensional hypercube with dimensions $1/m_T \times 1/m_T \times \ldots \times 1/m_T$.
\marrew{
We use index $p$ to denote a set in ${\cal P}_T$.
For learner $i$ let $p_i(t)$ be the set in ${\cal P}_T$ which $x_i(t)$ belongs to.\footnote{If $x_i(t)$ is an element of the boundary of multiple sets, then it is randomly assigned to one of these sets.}
}

%
%


\begin{figure}[htb]
\add{\vspace{-0.1in}}
\fbox {
\begin{minipage}{0.95\columnwidth}
{\fontsize{9}{9}\selectfont
\flushleft{CLUP for learner $i$:}
\begin{algorithmic}[1]
\STATE{Input: $D_1(t)$, $D_2(t)$, $D_3(t)$, $T$, $m_T$}
\STATE{Initialize sets: Create partition ${\cal P}_T$ of $[0,1]^D$ into $(m_T)^D$ identical hypercubes}
\STATE{Initialize counters:
$N^i_p =0$, $\forall p \in {\cal P}_T$,
$N^i_{k,p}=0, \forall k \in {\cal K}_i, p \in {\cal P}_T$, $N^{\textrm{tr}, i}_{j, p} =0, \forall j \in {\cal M}_{-i}, p \in {\cal P}_T$
}
\STATE{Initialize estimates: $\bar{r}^i_{k,p} =0$, $\forall k \in {\cal K}_i$, $p \in {\cal P}_T$
}
\WHILE{$t \geq 1$}
\STATE{Run CLUPmax to get choice $a_i$, $p=p_i(t)$ and $train$}
\STATE{If $a_i \in {\cal M}_{-i}$ call learner $a_i$ and pass $x_i(t)$}
\STATE{Receive ${\cal C}_i(t)$, the set of learners who called $i$, and their contexts}
\IF{${\cal C}_i(t) \neq \emptyset$}
\STATE{Run CLUPcoop to get arms to be selected $\boldsymbol{b}_i := \{ b_{i,j}  \}_{j \in  {\cal C}_i(t)}$ and sets that the contexts lie in $\boldsymbol{p}_i := \{ p_{i,j}  \}_{j \in  {\cal C}_i(t)}$}
\ENDIF
\IF{$a_i \in {\cal F}_i$}
\STATE{Pay cost $d^i_{a_i}$, receive random reward $r$ drawn from $G_{a_i}(x_i(t))$}
\ELSE
\STATE{Pay cost $d^i_{a_i}$, receive random reward $r$ drawn from $G_{b_{a_i,i}}(x_i(t))$}
\ENDIF
\IF{$train = 1$}
\STATE{$N^{\textrm{tr}, i}_{a_i, p} ++$}
\ELSE
\STATE{$\bar{r}^i_{\tsp{a_i},p} = \frac{\bar{r}^i_{\tsp{a_i},p} N^i_{a_i,p}  + r}{N^i_{a_i,p}+1}$}
\STATE{$N^i_{p}++$, $N^i_{a_i,p} ++$}
\ENDIF
\IF{${\cal C}_i(t) \neq \emptyset$}
\FOR{$j \in {\cal C}_i(t)$}
\STATE{Observe random reward $r$ drawn from $G_{b_{i,j}}(x_j(t))$}
\STATE{$\bar{r}^i_{b_{i,j},p_{i,j}} = \frac{\bar{r}^i_{b_{i,j},p_{i,j}} N^i_{b_{i,j},p_{i,j}}  + r}{N^i_{b_{i,j},p_{i,j}}+1}$}
\STATE{$N^i_{p_{i,j}}++$, $N^i_{b_{i,j},p_{i,j}}++$}
\ENDFOR
\ENDIF
\STATE{$t=t+1$}
\ENDWHILE
\end{algorithmic}
}
\end{minipage}
} \caption{Pseudocode for CLUP algorithm.} \label{fig:CLUP}
\add{\vspace{-0.12in}}
\end{figure}

\begin{figure}[htb]
\add{\vspace{-0.1in}}
\fbox {
\begin{minipage}{0.95\columnwidth}
{\fontsize{9}{7}\selectfont
\flushleft{CLUPmax (maximization part of CLUP) for learner $i$:}
\begin{algorithmic}[1]
\STATE{$train=0$}
\STATE{Find the set in ${\cal P}_T$ that $x_i(t)$ belongs to, i.e., $p_i(t)$}
\STATE{Let $p = p_i(t)$}
\STATE{Compute the set of under-explored arms ${\cal F}^{\textrm{ue}}_{i,p}(t)$ given in (\ref{eqn:underexploreda})}
\IF{${\cal F}^{\textrm{ue}}_{i,p}(t) \neq \emptyset$}
\STATE{Select $a_i$ randomly from ${\cal F}^{\textrm{ue}}_{i,p}(t)$}
\ELSE
\STATE{Compute the set of training candidates ${\cal M}^{\textrm{ct}}_{i,p}(t)$ given in (\ref{eqn:traincand})}
\STATE{//Update the counters of training candidates}
\FOR{$j \in {\cal M}^{\textrm{ut}}_{i,p}(t)$}
\STATE{Obtain $N^j_p$ from learner $j$, set $N^{\textrm{tr}, i}_{j,p} = N^j_p - N^i_{j,p}$}
\ENDFOR
\STATE{Compute the set of under-trained learners ${\cal M}^{\textrm{ut}}_{i,p}(t)$ given in (\ref{eqn:undertrained})}
\STATE{Compute the set of under-explored learners ${\cal M}^{\textrm{ue}}_{i,p}(t)$ given in (\ref{eqn:underexploredl})}
\IF{${\cal M}^{\textrm{ut}}_{i,p}(t) \neq \emptyset$}
\STATE{Select $a_i$ randomly from ${\cal M}^{\textrm{ut}}_{i,p}(t)$, $train=1$}
\ELSIF{${\cal M}^{\textrm{ue}}_{i,p}(t) \neq \emptyset$}
\STATE{Select $a_i$ randomly from ${\cal M}^{\textrm{ue}}_{i,p}(t)$}
\ 
\ELSE
\STATE{Select $a_i$ randomly from $\argmax_{k \in {\cal K}_i} \bar{r}^i_{k,p} - d^i_k$}
\ENDIF
\ENDIF
\end{algorithmic}
}
\end{minipage}
} \caption{Pseudocode for the maximization part of CLUP algorithm.} \label{fig:CLUPmax}
\add{\vspace{-0.12in}}
\end{figure}

\begin{figure}[htb]
\add{\vspace{-0.1in}}
\fbox {
\begin{minipage}{0.95\columnwidth}
{\fontsize{9}{7}\selectfont
\flushleft{CLUPcoop (cooperation part of CLUP) for learner $i$:}
\begin{algorithmic}[1]
\FOR{$j \in {\cal C}_i(t)$}
\STATE{Find the set in ${\cal P}_T$ that $x_j(t)$ belongs to, i.e., $p_{i,j}$}
\STATE{Compute the set of under-explored arms ${\cal F}^{\textrm{ue}}_{i,p_{i,j}}(t)$ given in (\ref{eqn:underexploreda})}
\IF{${\cal F}^{\textrm{ue}}_{i,p_{i,j}}(t) \neq \emptyset$}
\STATE{Select $b_{i,j}$ randomly from ${\cal F}^{\textrm{ue}}_{i,p_{i,j}}(t)$}
\ELSE
\STATE{$b_{i,j} = \argmax_{f \in {\cal F}_i} \bar{r}^i_{f, p_{i,j}}$}
\ENDIF
\ENDFOR
\end{algorithmic}
}
\end{minipage}
} \caption{Pseudocode for the cooperation part of CLUP algorithm.} \label{fig:CLUPcoop}
\add{\vspace{-0.12in}}
\end{figure}

\marrew{First, we will describe the maximization part of CLUP.}
At time slot $t$ learner $i$ can be in one of the three phases: {\em training} phase in which learner $i$ calls another learner with its context such that when the reward is received, the called learner can update the estimated reward of its selected arm (but learner $i$ does not update the estimated reward of the selected learner), 
{\em exploration} phase in which learner $i$ selects a choice in ${\cal K}_i$ and updates its estimated reward,
and {\em exploitation} phase in which learner $i$ selects the choice with the highest estimated net reward.

\marrew{Recall that the learners are cooperative. Hence, when called by another learner, learner $i$ will choose its arm with the highest estimated reward for the calling learner's context.}
\marrew{
To gain the highest possible reward in exploitations,
learner $i$ must have an accurate estimate of other learners' expected rewards without observing the arms selected by them. 
In order to do this, before forming estimates about the expected reward of learner $j$, learner $i$ needs to make sure that learner $j$ will almost always select its best arm when called by learner $i$.
Thus, the training phase of learner $i$ helps other learners build accurate estimates about rewards of their arms, before learner $i$ uses any rewards from these learners to form reward estimates about them.
In contrast, the exploration phase of learner $i$ helps it to build accurate estimates about rewards of its choices. 
These two phases indirectly help learner $i$ to maximize its total expected reward in the long run.
}

\marrew{
Next, we define the counters learner $i$ keeps for each set in ${\cal P}_T$ for each choice in ${\cal K}_i$, which are used to decide its current phase.}
Let $N^i_p(t)$ be the number of context arrivals to learner $i$ in $p \in {\cal P}_T$ by time $t$ (its own arrivals and arrivals to other learners who call learner $i$) except the training phases of learner $i$.
For $f \in {\cal F}_i$, let $N^i_{f,p}(t)$ be the number of times arm $f$ is selected in response to a context arriving to set $p$ by learner $i$ by time $t$ (including times other learners select learner $i$ for their contexts in set $p$).
Other than these, learner $i$ keeps two counters for each other learner in each set in the partition, which it uses to decide training, exploration or exploitation.
\marrew{
The first one, i.e., $N^{\textrm{tr}, i}_{j,p}(t)$, is an estimate on the number of context arrivals to learner $j$ from all learners except the training phases of learner $j$ and exploration, exploitation phases of learner $i$.
This is an estimate because learner $i$ updates this counter only when it needs to train learner $j$.
}
The second one, i.e., $N^i_{j,p}(t)$, counts the number of context arrivals to learner $j$ only from the contexts of learner $i$ in set $p$ at times learner $i$ selected learner $j$ in its exploration and exploitation phases by time $t$.
Based on the values of these counters at time $t$, learner $i$ either trains, explores or exploits a choice in ${\cal K}_i$. 
This three-phase learning structure is one of the major components of our learning algorithm which makes it different than the algorithms proposed for the contextual bandits in the literature which assigns an index to each choice and selects the choice with the highest index. 

At each time slot $t$, learner $i$ first identifies $p_i(t)$.
Then, it chooses its phase at time $t$ by giving highest priority to exploration of its own arms, second highest priority to training of other learners, third highest priority to exploration of other learners, and lowest priority to exploitation. 
The reason that exploration of own arms has a higher priority than training of other learners is that it can reduce the number of trainings required by other learners, which we will describe below.

First, learner $i$ identifies its set of under-explored arms:
\begin{align}
{\cal F}^{\textrm{ue}}_{i,p}(t) &:= \{ f \in {\cal F}_{i}: N^i_{f,p}(t) \leq D_1(t) \}
\label{eqn:underexploreda}
\end{align}
where $D_1(t)$ is a deterministic, increasing function of $t$ which is called {\em the control function}.
We will specify this function later, when analyzing the regret of CLUP.
The accuracy of reward estimates of learner $i$ for its own arms increases with $D_1(t)$, hence it should be selected to balance the tradeoff between accuracy and the number of explorations.
If this set is non-empty, learner $i$ enters the exploration phase and randomly selects an arm in this set to explore it.
Otherwise, learner $i$ identifies the set of training candidates:
\begin{align}
{\cal M}^{\textrm{ct}}_{i,p}(t) := \{ j \in {\cal M}_{-i}: N^{\textrm{tr}, i}_{j,p}(t) \leq D_2(t)\} \label{eqn:traincand}
\end{align}
where $D_2(t)$ is a control function similar to $D_1(t)$.
Accuracy of other learners' reward estimates of their own arms increase with $D_2(t)$, hence it should be selected to balance the possible reward gain of learner $i$ due to this increase with the reward loss of learner $i$ due to number of trainings.
If this set is non-empty, learner $i$ asks the learners $j \in {\cal M}^{\textrm{ct}}_{i,p}(t)$ to report $N^j_p(t)$.
Based in the reported values it recomputes $N^{\textrm{tr},i}_{j,p}(t)$ as $N^{\textrm{tr},i}_{j,p}(t) = N^j_p(t) - N^i_{j,p}(t)$.
Using the updated values, learner $i$ identifies the set of under-trained learners:
\begin{align}
{\cal M}^{\textrm{ut}}_{i,p}(t) := \{ j \in {\cal M}_{-i}: N^{\textrm{tr}, i}_{j,p}(t) \leq D_2(t)\}. \label{eqn:undertrained}
\end{align}
If this set is non-empty, learner $i$ enters the training phase and randomly selects a learner in this set to train it.\footnote{Most of the regret bounds proposed in this paper can also be achieved by setting $N^{\textrm{tr}, i}_{j,p}(t)$ to be the number of times learner $i$ trains learner $j$ by time $t$, without considering other context observations of learner $j$. 
However, by recomputing $N^{\textrm{tr}, i}_{j,p}(t)$, learner $i$ can avoid many unnecessary trainings especially when own context arrivals of learner $j$ is adequate for it to form accurate estimates about its arms for  set $p$ or when learners other than learner $i$ have already helped learner $j$ to build accurate estimates for its arms in set $p$.}
When ${\cal M}^{\textrm{ct}}_{i,p}(t)$ or ${\cal M}^{\textrm{ut}}_{i,p}(t)$ is empty, this implies that there is no under-trained learner, hence learner $i$ checks if there is an under-explored choice. 
The set of learners that are under-explored by learner $i$ is given by 
\begin{align}
{\cal M}^{\textrm{ue}}_{i,p}(t) &:=  \{ j \in {\cal M}_{-i}: N^i_{j,p}(t) \leq D_3(t) \}
\label{eqn:underexploredl}
\end{align}
where $D_3(t)$ is also a control function similar to $D_1(t)$.
If this set is non-empty, learner $i$ enters the exploration phase and randomly selects a choice in this set to explore it.
Otherwise, learner $i$ enters the exploitation phase in which it selects the choice with the highest estimated net reward, i.e.,
\begin{align}
a_i(t) \in \argmax_{k \in {\cal K}_i} \bar{r}^i_{k,p}(t) - d^i_k \label{eqn:maximizer}
\end{align}
where $\bar{r}^i_{k,p}(t)$ is the sample mean estimate of the rewards learner $i$ observed (not only collected) from choice $k$ by time $t$, which is computed as follows. 
For $j \in {\cal M}_{-i}$, let ${\cal E}^i_{j,p}(t)$ be the set of rewards collected by learner $i$ at times it selected learner $j$ while learner $i$'s context is in set $p$ in its exploration and exploitation phases by time $t$.
For estimating the rewards of its own arms, learner $i$ can also use the rewards obtained by other learner at times they called learner $i$.
\marrew{
In order to take this into account, for $f \in {\cal F}_i$, let ${\cal E}^i_{f,p}(t)$ be the set of rewards collected by learner $i$ at times it selected its arm $f$ for its own contexts in set $p$ union the set of rewards observed by learner $i$ when it selected its arm $f$ for other learners calling it with contexts in set $p$ by time $t$.
}
Therefore, sample mean reward of choice $k \in {\cal K}_i$ in set $p$ for learner $i$ is defined as
%
$\bar{r}^i_{k,p}(t) = (\sum_{r \in {\cal E}^i_{k,p}(t)} r)/|{\cal E}^i_{k,p}(t)|$.
%
\marrew{
An important observation is that computation of $\bar{r}^i_{k,p}(t)$ does not take into account the costs related to selecting choice $k$. 
Reward generated by an arm only depends on the context it is selected at but not on the identity of the learner for whom that arm is selected. 
However, the costs incurred depend on the identity of the learner.
}
Let $\hat{\mu}^i_{k,p}(t) := \bar{r}^i_{k,p}(t) - d^i_k$ be the estimated net reward of choice $k$ for set $p$.
Of note, when there is more than one maximizer of (\ref{eqn:maximizer}), one of them is randomly selected.
In order to run CLUP, learner $i$ does not need to keep the sets ${\cal E}^i_{k,p}(t)$ in its memory. 
$\bar{r}^i_{k,p}(t)$ can be computed by using only $\bar{r}^i_{k,p}(t-1)$ and the reward at time $t$.

The cooperation part of CLUP operates as follows.
Let ${\cal C}_i(t)$ be the learners who call learner $i$ at time $t$. 
For each $j \in {\cal C}_i(t)$, learner $i$ first checks if it has any under-explored arm $f$ for $p_j(t)$, i.e., $f$ such that $N^i_{f,p_j(t)}(t) \leq D_1(t)$. If so, it randomly selects one of its under-explored arms and provides its reward to learner $j$. Otherwise, it exploits its arm with the highest estimated reward for learner $j$'s context, i.e.,
\begin{align}
b_{i,j}(t) \in \argmax_{f \in {\cal F}_i} \bar{r}^i_{f,p_j(t)}(t). \label{eqn:maximizer2}
\end{align}

\subsection{Analysis of the Regret of CLUP}

Let $\beta_a := \sum_{t=1}^{\infty} 1/t^a$, and let $\log(.)$ denote logarithm in base $e$.
For each set (hypercube) $p \in {\cal P}_T$ let 
%
$\overline{\pi}_{f,p} := \sup_{x \in p} \pi_f(x)$,
$\underline{\pi}_{f,p} := \inf_{x \in p} \pi_f(x)$, for $f \in {\cal F}$, and
$\overline{\mu}^i_{k,p} := \sup_{x \in p} \mu^i_k(x)$,
$\underline{\mu}^i_{k,p} := \inf_{x \in p} \mu^i_k(x)$, for $k \in {\cal K}_i$.
%
Let $x^*_p$ be the context at the center (center of symmetry) of the hypercube $p$. We define the optimal choice of learner $i$ for set $p$ as
%
$k^*_i(p) := \argmax_{k \in {\cal K}_i} \mu^i_k(x^*_p)$.
%
When the set $p$ is clear from the context, we will simply denote the optimal choice for set $p$ with $k^*_i$.
Let
\begin{align*}
{\cal L}^i_p(t) := \left\{ k \in {\cal K}_i \textrm{ such that }  \underline{\mu}^i_{k^*_i(p),p} - \overline{\mu}^i_{k,p} > A t^{\theta} \right\}
\end{align*}
be the set of suboptimal choices for learner $i$ for hypercube $p$ at time $t$, where 
$\theta<0$, $A > 0$ are parameters that are only used in the analysis of the regret and do not need to be known by the learners.
First, we will give regret bounds that depend on values of $\theta$ and $A$ and then we will optimize over these values to find the best bound.
Also related to this let
\begin{align*}
{\cal F}^j_p(t) := \left\{ f \in {\cal F}_j \textrm{ such that }  \underline{\pi}_{f^*_j(p),p} - \overline{\pi}_{f,p} > A t^{\theta} \right\}
\end{align*}
be the set of suboptimal arms of learner $j$ for hypercube $p$ at time $t$, where $f^*_j(p) = \argmax_{f \in {\cal F}_j} \pi_f(x^*_p)$. Also when the set $p$ is clear from the context we will just use $f^*_j$.
The arms in ${\cal F}^j_p(t) $ are the ones that learner $j$ should not select when called by another learner.

The regret given in (\ref{eqn:opt2}) can be written as a sum of three components: $R_i(T) = \mathrm{E} [R^e_i(T)] + \mathrm{E} [R^s_i(T)] + \mathrm{E} [R^n_i(T)]$,
where $R^e_i(T)$ is the regret due to trainings and explorations by time $T$, $R^s_i(T)$ is the regret due to suboptimal choice selections in exploitations by time $T$ and $R^n_i(T)$ is the regret due to near optimal choice selections in exploitations by time $T$, which are all random variables. In the following lemmas we will bound each of these terms separately. The following lemma bounds $\mathrm{E} [R^e_i(T)]$. 
%
%
\begin{lemma} \label{lemma:explorations}
When CLUP is run by all learners with parameters $D_1(t) = t^{z} \log t$, $D_2(t) = F_{\max} t^{z} \log t$, $D_3(t) = t^{z} \log t$ and $m_T = \left\lceil T^{\gamma} \right\rceil$,\footnote{For a number $r \in \mathbb{R}$, let $\lceil r  \rceil$ be the smallest integer that is greater than or equal to $r$.} where $0<z<1$ and $0<\gamma<1/D$, we have
\add{\vspace{-0.1in}}
\begin{align*}
\mathrm{E} [R^e_i(T)] &\leq  \sum_{p=1}^{(m_T)^D} \tsp{2} (F_i + M_i (F_{\max} + 1)) T^{z} \log T \\
&+ \tsp{2}  (F_i + 2 M_i ) (m_T)^D \\
&\leq \tsp{2^{D+1}} \tsp{Z_i} T^{z+\gamma D} \log T 
+ \tsp{2^{D+1}} ( F_i+ 2 M_i ) T^{\gamma D} 
\end{align*}
where
\begin{align}
\tsp{Z_i :=  (F_i+ M_i (F_{\max} + 1))} .   \label{eqn:Zi}
\end{align}
\end{lemma}
\remove{
\begin{proof}
Since time slot $t$ is a training or an exploration slot for learner $i$ if and only if 
${\cal M}^{\textrm{ut}}_{i,p_i(t)}(t) \cup {\cal M}^{\textrm{ue}}_{i,p_i(t)}(t) \cup {\cal F}^{\textrm{ue}}_{i,p_i(t)}(t)  \neq \emptyset$, up to time $T$, there can be at most   $\left\lceil T^{z} \log T \right\rceil$ exploration slots in which an arm in $f \in {\cal F}_i$ is selected by learner $i$, 
$\left\lceil F_{\max} T^{z} \log T \right\rceil$ training slots in which learner $i$ selects learner $j \in {\cal M}_{-i}$, $\left\lceil T^{z} \log T \right\rceil$ exploration slots in which learner $i$ selects learner $j \in {\cal M}_{-i}$. 
\tsp{Since $\mu^i_k(x) = \pi^i_k(x) - d^i_k \in [-1,1]$ for all $k \in {\cal K}_i$, the realized (hence expected) one slot loss due to any choice is bounded above by $2$.}
\tsp{Hence, the result follows from summing the above terms and multiplying by $2$, and the fact that $(m_T)^D \leq 2^D T^{\gamma D}$ for any $T \geq 1$.}
\end{proof}
}

From Lemma \ref{lemma:explorations}, we see that the regret due to explorations is linear in the number of hypercubes $(m_T)^D$, hence exponential in parameter $\gamma$ and $z$. 

For any $k \in {\cal K}_i$ and $p \in {\cal P}_T$, the sample mean $\bar{r}^i_{k,p}(t)$ represents a random variable which is the average of the independent samples in set ${\cal E}^i_{k,p}(t)$.
Let $\Xi^i_{j,p}(t)$ be the event that a suboptimal arm $f \in {\cal F}_{j}$ is selected by learner $j \in {\cal M}_{-i}$, when it is called by learner $i$ for a context in set $p$ for the $t$th time in the exploitation phases of learner $i$.
Let $X^i_{j,p}(t)$ denote the random variable which is the number of times learner $j$ selects a suboptimal arm when called by learner $i$ in exploitation slots of learner $i$ when the context is in set $p \in {\cal P}_T$ by time $t$.
Clearly, we have
\begin{align}
X^i_{j,p}(t) = \sum_{t'=1}^{|{\cal E}^i_{j,p}(t)|} \mathrm{I}(\Xi^i_{j,p}(t')) \label{eqn:thisxi}
\end{align}
where $\mathrm{I}(\cdot)$ is the indicator function which is equal to $1$ if the event inside is true and $0$ otherwise. The following lemma bounds $\mathrm{E} [R^s_i(T)]$.
\begin{lemma} \label{lemma:suboptimal1}
Consider all learners running CLUP with parameters $D_1(t) = t^{z} \log t$, $D_2(t) = F_{\max} t^{z} \log t$, $D_3(t) = t^{z} \log t$ and $m_T = \left\lceil T^{\gamma} \right\rceil$, where $0<z<1$ and $0<\gamma<1/D$. For any $0<\phi<1$ if
$t^{-z/2} + t^{\phi -z} + L D^{\alpha/2} t^{-\gamma \alpha}  \leq A t^\theta/2$ holds for all $t \leq T$, then we have
\begin{align*}
\mathrm{E} [R^s_i(T)] 
&\leq  4 (M_i + F_i) \beta_2 + 4 (M_i + F_i) M_i F_{\max} \beta_2 \frac{T^{1-\phi}}{1-\phi} .
\end{align*}
\end{lemma}

\begin{proof}
Consider time $t$. Let 
${\cal W}^i(t) := \{ {\cal M}^{\textrm{ut}}_{i,p_i(t)}(t) \cup {\cal M}^{\textrm{ue}}_{i,p_i(t)}(t) \cup {\cal F}^{\textrm{ue}}_{i,p_i(t)}(t)  = \emptyset  \}$
be the event that learner $i$ exploits at time $t$.

First, we will bound the probability that learner $i$ selects a suboptimal choice in an exploitation slot.
Then, using this we will bound the expected number of times a suboptimal choice is selected by learner $i$ in exploitation slots.
Note that every time a suboptimal choice is selected by learner $i$, since $\mu^i_k(x) = \pi^i_k(x) - d^i_k \in [-1,1]$ for all $k \in {\cal K}_i$, the realized (hence expected) loss is bounded above by $2$.
Therefore, $2$ times the expected number of times a suboptimal choice is selected in an exploitation slot bounds $\mathrm{E} [R^s_i(T)]$.
Let ${\cal V}^i_{k}(t)$ be the event that choice $k$ is chosen at time $t$ by learner $i$. We have
%
\ntsp{$R^s_i(T) \leq \tsp{2} \sum_{t=1}^T \sum_{k \in {\cal L}^i_{p_i(t)}(t) } 
\mathrm{I}(  {\cal V}^i_{k}(t), {\cal W}^i(t) )$.}
Adopting the standard probabilistic notation, for two events $E_1$ and $E_2$, $\mathrm{I}(E_1, E_2)$ is equal to $\mathrm{I}(E_1 \cap E_2)$.
%
Taking the expectation
\begin{align}
\ntsp{ \mathrm{E}[R^s_i(T)] 
\leq 2 \sum_{t=1}^T \sum_{k \in {\cal L}^i_{p_i(t)}(t)} \mathrm{P}({\cal V}^i_{k}(t), {\cal W}^i(t) ).} \label{eqn:subregret}
\end{align}
Let ${\cal B}^i_{j,p}(t)$ be the event that at most $t^{\phi}$ samples in ${\cal E}^i_{j,p}(t)$ are collected from suboptimal arms of learner $j$ in hypercube $p$.
\ntsp{Let ${\cal B}^i(t) := \bigcap_{j \in {\cal M}_{-i}} {\cal B}^i_{j,p_i(t)}(t)$.}
For a set ${\cal A}$, let ${\cal A}^c$ denote the complement of that set.
For any $k \in {\cal L}^i_{p_i(t)}(t)$, we have
\begin{align}
& \{ {\cal V}^i_{k}(t), {\cal W}^i(t)\} 
\subset \left\{ \hat{\mu}^i_{k,p_i(t)}(t) \geq \hat{\mu}^i_{k^*_i,p_i(t)}(t), {\cal W}^i(t), {\cal B}^i(t) \right\}  \notag \\
& \cup \left\{ \hat{\mu}^i_{k,p_i(t)}(t) \geq \hat{\mu}^i_{k^*_i,p_i(t)}(t), {\cal W}^i(t), {\cal B}^i(t)^c \right\} \notag \\
&\subset \left\{ \hat{\mu}^i_{k,p_i(t)}(t) \geq \overline{\mu}^i_{k,p_i(t)} + H_t, {\cal W}^i(t), {\cal B}^i(t)  \right\} \notag \\
&\cup \left\{ \hat{\mu}^i_{k^*_i,p_i(t)}(t) \leq \underline{\mu}^i_{k^*_i,p_i(t)} - H_t, {\cal W}^i(t), {\cal B}^i(t)  \right\} \notag \\
& \cup \left\{ \hat{\mu}^i_{k,p_i(t)}(t) \geq \hat{\mu}^i_{k^*_i,p_i(t)}(t), 
\hat{\mu}^i_{k,p_i(t)}(t) < \overline{\mu}^i_{k,p_i(t)} + H_t, \right. \notag \\
& \left. \hat{\mu}^i_{k^*_i,p_i(t)}(t) > \underline{\mu}^i_{k^*_i,p_i(t)} - H_t,
{\cal W}^i(t) ,{\cal B}^i(t)  \right\}  \notag \\
& \cup \{ {\cal B}^i(t)^c , {\cal W}^i(t) \}  
\end{align}
for some $H_t >0$. This implies that 
\begin{align}
& \mathrm{P} \left( {\cal V}^i_{k}(t), {\cal W}^i(t) \right) \notag \\
& \leq \mathrm{P} \left( \hat{\mu}^i_{k,p_i(t)}(t) \geq \overline{\mu}^i_{k,p_i(t)} + H_t, {\cal W}^i(t), {\cal B}^i(t)  \right) \notag  \\
&+ \mathrm{P} \left( \hat{\mu}^i_{k^*_i,p_i(t)}(t) \leq \underline{\mu}^i_{k^*_i,p_i(t)} - H_t, {\cal W}^i(t), {\cal B}^i(t) \right) \notag \\
&+ \mathrm{P} \left( \hat{\mu}^i_{k,p_i(t)}(t) \geq \hat{\mu}^i_{k^*_i,p_i(t)}(t), 
\hat{\mu}^i_{k,p_i(t)}(t) < \overline{\mu}^i_{k,p_i(t)} + H_t, \right. \notag \\
& \left. \hat{\mu}^i_{k^*_i,p_i(t)}(t) > \underline{\mu}^i_{k^*_i,p_i(t)} - H_t,
{\cal W}^i(t), {\cal B}^i(t)  \right) \notag \\
&+ \mathrm{P} ({\cal B}^i(t)^c, {\cal W}^i(t) ) . \notag 
\end{align}
Since for any $k \in {\cal K}$, $\bar{\mu}^i_{k,p_i(t)} = \sup_{x \in p_i(t)} \mu^i_k(x)$, we have for any suboptimal choice $k \in {\cal L}^i_{p_i(t)}(t)$,
\begin{align}
& \mathrm{P} \left( \hat{\mu}^i_{k,p_i(t)}(t) \geq \overline{\mu}^i_{k,p_i(t)} + H_t, {\cal W}^i(t), {\cal B}^i(t)  \right) \notag \\
&\leq \exp(-2 H^2_t t^z \log t)  \label{eqn:r1new1}
\end{align}
by Chernoff-Hoeffding bound since on event ${\cal W}^i(t)$ at least $t^z \log t$ samples are taken from each choice.
Similarly, we have
\begin{align}
& \mathrm{P} \left( \hat{\mu}^i_{k^*_i,p_i(t)}(t) \leq \underline{\mu}^i_{k^*_i,p_i(t)} - H_t, {\cal W}^i(t), {\cal B}^i(t) \right) \notag \\
&\leq \exp(- 2 ( H_t - t^{\phi - z} - L D^{\alpha/2} m_T^{-\alpha} )^2 t^z \log t) \label{eqn:r1new2}
\end{align}
which follows from the fact that the maximum variation of expected rewards within $p_i(t)$ is at most $L D^{\alpha/2} m_T^{-\alpha}$ and on event ${\cal B}^i(t)$ at most $t^{\phi}$ observations from any choice comes from a suboptimal arm of the learner corresponding to that choice. 
For $k \in {\cal L}^i_{p_i(t)}(t)$, when
\begin{align}
2 H_t \leq A t^{\theta}
\label{eqn:boundcond}
\end{align}
the three inequalities given below
\begin{align*}
\underline{\mu}_{k^*_i,p_i(t)} - \overline{\mu}^i_{k,p_i(t)} &> A t^{\theta}  \\
\hat{\mu}^{i}_{k,p_i(t)}(t) &< \overline{\mu}^i_{k,p_i(t)} + H_t \\
\hat{\mu}^{i}_{k^*_i,p_i(t)}(t) &> \underline{\mu}^i_{k^*_i,p_i(t)}  - H_t
\end{align*}
together imply that 
%
$\hat{\mu}^{i}_{k,p_i(t)}(t) < \hat{\mu}^{i}_{k^*_i,p_i(t)}(t)$,
%
which implies that
\begin{align}
& \mathrm{P} \left( \hat{\mu}^i_{k,p_i(t)}(t) \geq \hat{\mu}^i_{k^*_i,p_i(t)}(t), 
\hat{\mu}^i_{k,p_i(t)}(t) < \overline{\mu}^i_{k,p_i(t)} + H_t, \right. \notag \\
& \left. \hat{\mu}^i_{k^*_i,p_i(t)}(t) > \underline{\mu}^i_{k^*_i,p_i(t)} - H_t,
{\cal W}^i(t), {\cal B}^i(t)  \right) = 0. \label{eqn:vktbound1}
\end{align}
Using the results of (\ref{eqn:r1new1}) and (\ref{eqn:r1new2}) and by setting 
\begin{align}
H_t &= t^{-z/2} + t^{\phi-z} + L D^{\alpha/2} t^{-\gamma \alpha}  \label{eqn:thisht} \\
&\geq t^{-z/2} + t^{\phi-z} + L D^{\alpha/2} m_T^{-\alpha}      \notag
\end{align}
we get
 \begin{align}
 & \mathrm{P} \left( \hat{\mu}^i_{k,p_i(t)}(t) \geq \overline{\mu}^i_{k,p_i(t)} + H_t, {\cal W}^i(t), {\cal B}^i(t)  \right) \leq t^{-2} \label{eqn:vktbound2}
 \end{align}
 and
\begin{align}
& \mathrm{P} \left( \hat{\mu}^i_{k^*_i,p_i(t)}(t) \leq \underline{\mu}^i_{k^*_i,p_i(t)} - H_t, {\cal W}^i(t), {\cal B}^i(t) \right)  \leq t^{-2} . \label{eqn:vktbound3}
\end{align} 
All that is left is to bound $\mathrm{P} ({\cal B}^i(t)^c, {\cal W}^i(t))$. Applying the union bound, we have 
\begin{align}
\mathrm{P} ({\cal B}^i(t)^c, {\cal W}^i(t)) 
\leq \sum_{j \in {\cal M}_{-i}} \mathrm{P} ( {\cal B}^i_{j,p_i(t)}(t)^c, {\cal W}^i(t) ) .      \notag
\end{align}
We have $\{ {\cal B}^i_{j,p_i(t)}(t)^c, {\cal W}^i(t)  \} = \{ X^i_{j,p_i(t)}(t) \geq t^\phi \}$ (Recall $ X^i_{j,p_i(t)}(t)$ from (\ref{eqn:thisxi})).
Applying the Markov inequality we have
%
$\mathrm{P} ({\cal B}^i_{j,p_i(t)}(t)^c, {\cal W}^i(t)) \leq \mathrm{E} [X^i_{j,p_i(t)}(t)]/t^\phi$.
%
Recall that
%
$X^i_{j,p_i(t)}(t) = \sum_{t'=1}^{|{\cal E}^i_{j,p_i(t)}(t)|} \mathrm{I}(\Xi^i_{j,p_i(t)}(t'))$,
%
and
\begin{align*}
& \mathrm{P} \left( \Xi^i_{j,p_i(t)}(t) \right) 
\leq \sum_{m \in {\cal F}^j_{p_i(t)} (t) } \mathrm{P} 
\left( \bar{r}^j_{m,p_i(t)}(t) \geq \bar{r}^{j}_{f^*_j, p_i(t)}(t) \right) \\
&\leq \sum_{m \in  {\cal F}^j_{p_i(t)}(t) }
\left(  \mathrm{P} \left( 
\bar{r}^j_{m,p_i(t)}(t) \geq \overline{\pi}_{m,p_i(t) } + H_t, {\cal W}^i(t) \right) \right. \\  
& \left. + \mathrm{P} \left( 
\bar{r}^{j}_{f^*_j, p_i(t)}(t) \leq \underline{\pi}_{f^*_j,p_i(t)} - H_t, {\cal W}^i(t) \right) \right.  \\
&\left. + \mathrm{P} \left( \bar{r}^j_{m,p_i(t)}(t) \geq \bar{r}^j_{f^*_j,p_i(t)}(t), 
\bar{r}^j_{m,p_i(t)}(t) < \overline{\pi}_{m,p_i(t)} + H_t, \right. \right. \\
& \left. \left.  \bar{r}^{j}_{f^*_j,p_i(t)}(t) > \underline{\pi}_{f^*_j,p_i(t)} - H_t ,
{\cal W}^i(t) \right)  \right).
\end{align*}
When (\ref{eqn:boundcond}) holds, the last probability in the sum above is equal to zero while the first two probabilities are upper bounded by $e^{-2(H_t)^2 t^z \log t}$.
This is due to the training phase of CLUP by which it is guaranteed that every learner samples each of its own arms at least $t^z \log t$ times before learner $i$ starts forming estimates about learner $j$.
Therefore for any $p \in {\cal P}_T$, we have
%
$\mathrm{P} \left( \Xi^i_{j,p}(t) \right) \leq \sum_{m \in  {\cal F}^j_p(t) } 2 e^{-2(H_t)^2 t^z \log t} \leq 2 \tsp{F_j} t^{-2}$ for the value of $H_t$ given in (\ref{eqn:thisht}).
%
These together imply that 
%
$\mathrm{E} [X^i_{j,p}(t)] \leq \sum_{t'=1}^{\infty} \mathrm{P} (\Xi^i_{j,p}(t')) \leq 2 \tsp{F_j} \sum_{t'=1}^\infty t^{-2}$.
%
Therefore from the Markov inequality we get
\begin{align}
\mathrm{P} ({\cal B}^i_{j,p}(t)^c, {\cal W}^i(t)) = \mathrm{P} (X^i_{j,p}(t) \geq t^\phi) \leq 
2 \tsp{F_j} \beta_2 t^{-\phi}  \notag
\end{align}
for any $p \in {\cal P}_T$ and hence, 
\begin{align}
\mathrm{P} ({\cal B}^i(t)^c, {\cal W}^i(t)) \leq 2 M_i F_{\max} \beta_2 t^{-\phi} . \label{eqn:selectionbound}
\end{align}
Then, using (\ref{eqn:vktbound1}), (\ref{eqn:vktbound2}), (\ref{eqn:vktbound3}) and (\ref{eqn:selectionbound}), we have 
%
$\mathrm{P} \left( {\cal V}^i_{k}(t), {\cal W}^i(t)  \right) \leq 2 t^{-2} 
+ 2 M_i F_{\max} \beta_2 t^{-\phi}$,
%
for any $k \in {\cal L}^i_{p_i(t)}(t)$.
By (\ref{eqn:subregret}), and by the result of Appendix \ref{app:seriesbound}, we get the stated bound for $\mathrm{E} [R^s_i(T)]$.
%
%
%
\end{proof}


Each time learner $i$ calls learner $j$, learner $j$ selects one of its own arms in ${\cal F}_{j}$. There is a positive probability that learner $j$ will select one of its suboptimal arms, which implies that even if learner $j$ is near optimal for learner $i$, selecting learner $j$ may not yield a near optimal outcome. We need to take this into account, in order to bound $\mathrm{E}[R^n_i(T)]$. 
The next lemma bounds the expected number of such happenings.
\begin{lemma} \label{lemma:callother}
Consider all learners running CLUP with parameters $D_1(t) = t^{z} \log t$, $D_2(t) = F_{\max} t^{z} \log t$, $D_3(t) = t^{z} \log t$ and $m_T = \left\lceil T^{\gamma} \right\rceil$, where $0<z<1$ and $0<\gamma<1/D$. For any $0<\phi<1$ if
$t^{-z/2} + t^{\phi -z} + L D^{\alpha/2} t^{-\gamma \alpha}  \leq A t^\theta/2$ holds for all $t \leq T$, then we have
\begin{align*}
\mathrm{E} [X^i_{j,p}(t)] \leq 2 F_{\max} \beta_2
\end{align*}
for $j \in {\cal M}_{-i}$.
\end{lemma}
\begin{proof}
The proof is contained within the proof of the last part of Lemma 2. 
\end{proof}

We will use Lemma \ref{lemma:callother} in the following lemma to bound $\mathrm{E} [R^n_i(T)]$.

\begin{lemma} \label{lemma:nearoptimal}
Consider all learners running CLUP with parameters $D_1(t) = t^{z} \log t$, $D_2(t) = F_{\max} t^{z} \log t$, $D_3(t) = t^{z} \log t$ and $m_T = \left\lceil T^{\gamma} \right\rceil$, where $0<z<1$ and $0<\gamma<1/D$. For any $0<\phi<1$ if
$t^{-z/2} + t^{\phi -z} + L D^{\alpha/2} t^{-\gamma \alpha}  \leq A t^\theta/2$ holds for all $t \leq T$, then we have
\tsp{
\begin{align*}
\mathrm{E} [R^n_i(T)] & \leq \frac{2A}{1+\theta} T^{1+\theta} + 6 L D^{\alpha/2} T^{1-\alpha \gamma} \\
& + 4 M_i   F_{\max} \beta_2 2^D T^{\gamma D}. 
\end{align*}
}
\end{lemma}
\remove{
\begin{proof}
\tsp{
At any time $t$, for any $k \in  {\cal K}_i -  {\cal L}^i_p(t)$ and $x \in p$, we have 
$\mu^i_{ k^*_i(x) }(x) - \mu^i_k(x) \leq A t^{\theta} + 3 L D^{\alpha/2} T^{-\alpha \gamma}$. Similarly for any $j \in {\cal M}$, $f \in {\cal F}_j -  {\cal F}^j_p(t)$ and $x \in p$, we have 
$\pi_{ f^*_j(x) }(x) - \pi_f(x) \leq A t^{\theta} + 3 L D^{\alpha/2} T^{-\alpha \gamma}$.}

Let $p=p_i(t)$. Due to the above inequalities, if a near optimal arm in ${\cal F}_i \cap \tsp{ ({\cal K}_i -  {\cal L}^i_p(t))}$ is chosen by learner $i$ at time $t$, the contribution to the regret is at most 
\tsp{$A t^{\theta} + 3 L D^{\alpha/2} T^{-\alpha \gamma}$}.
If a near optimal learner $j \in {\cal M}_{-i} \cap \tsp{ ({\cal K}_i -  {\cal L}^i_p(t) )}$ is called by learner $i$ at time $t$, and if learner $j$ selects one of its near optimal arms in ${\cal F}_j - {\cal F}^j_p(t)$, then the contribution to the regret is at most $2 (A t^{\theta} + 3 L D^{\alpha/2} T^{-\alpha \gamma})$.
Therefore, the total regret due to near optimal choices of learner $i$ by time $T$ is upper bounded by 
\tsp{
\begin{align*}
2 \sum_{t=1}^T (A t^{\theta} + 3 L D^{\alpha/2} T^{-\alpha \gamma}) 
& \leq \frac{2A}{1+\theta} T^{1+\theta} + 6 L D^{\alpha/2} T^{1-\alpha \gamma}
\end{align*}
}by using the result in Appendix \ref{app:seriesbound}. 
Each time a near optimal learner in $j \in {\cal M}_{-i} \cap \tsp{ ({\cal K}_i -  {\cal L}^i_p(t) )}$ is called in an exploitation step, there is a small probability that the arm selected by learner $j$ is a suboptimal one.
Given in Lemma \ref{lemma:callother}, the expected number of times a suboptimal arm is chosen by learner $j$ for learner $i$ in each hypercube $p$ is bounded by $2 F_{\max}  \beta_2$. For each such choice, the one-slot regret of learner $i$ can be at most $2$, and the number of such hypercubes is bounded by $2^D T^{\gamma D}$.
\end{proof}
}


%% file: cos_analysis.tex
In the next theorem we bound the regret of learner $i$ by combining the above lemmas.
\begin{theorem}\label{theorem:cos}
Consider all learners running CLUP with parameters $D_1(t) = t^{2\alpha/(3\alpha+D)} \log t$, $D_2(t) = F_{\max} t^{2\alpha/(3\alpha+D)} \log t$, $D_3(t) = t^{2\alpha/(3\alpha+D)} \log t$ and $ m_T  = \left\lceil T^{1/(3\alpha + D)} \right\rceil$. Then, we have
\begin{align*}
& R_i(T) \leq 4 (M_i + F_i) \beta_2 \\
&+ T^{\frac{2\alpha+D}{3\alpha+D}}
\left( \frac{ 4 ( L D^{\alpha/2}+2) + 4 (M_i + F_i) M_i F_{\max} \beta_2 } {(2\alpha+D)/(3\alpha+D)} 
\right. \\ 
& \left. + 6 L D^{\alpha/2} +2^{D+1} Z_i \log T \right) \\
&+ T^{\frac{D}{3\alpha+D}} ( 2^{D+1} (F_i + 2 M_i  ) + 2^{D+2} M_i F_{\max} \beta_2).
\end{align*}
for any sequence of context arrivals $\{ x_i(t) \}_{t \in 1,\ldots,T}$, $i \in {\cal M}$.
Hence, $R_i(T) = \tilde{O} \left(M F_{\max} T^{\frac{2\alpha+D}{3\alpha+D}} \right)$, for all $i \in {\cal M}$,
where $Z_i$ is given in (\ref{eqn:Zi}).
\end{theorem}
\begin{proof}
The highest orders of regret that come from trainings, explorations, suboptimal and near optimal arm selections are $\tilde{O}(T^{\gamma D + z})$, $O(T^{1-\phi})$ and 
$O(T^{\max\{ 1-\alpha\gamma, 1+\theta \}} )$. We need to optimize them with respect to the constraint
%
$t^{-z/2} + t^{\phi -z} + L D^{\alpha/2} t^{-\gamma \alpha}  \leq A t^\theta/2$, $t \leq T$
which is assumed in Lemmas \ref{lemma:suboptimal1} and \ref{lemma:nearoptimal}.
The values that minimize the regret for which this constraint holds are $z = 2\alpha/(3\alpha+D)$, $\phi=z/2$, $\theta = -z/2$, $\gamma = z/(2\alpha)$ and $A = 2 L D^{\alpha/2} + 4$. 
Result follows from summing the bounds in Lemmas \ref{lemma:explorations}, \ref{lemma:suboptimal1} and \ref{lemma:nearoptimal}. 
\end{proof}
\begin{remark}
Although the parameter $m_T$ of CLUP depends on $T$ and hence we require $T$ as an input to the algorithm, we can make CLUP run independently of the final time $T$ and achieve the same regret bound by using a well known doubling trick (see, e.g., \cite{slivkins2009contextual}). Consider phases $\tau \in \{1,2,\ldots\}$, where each phase has length $2^{\tau}$. We run a new instance of algorithm CLUP at the beginning of each phase with time parameter $2^\tau$. Then, the regret of this algorithm up to any time $T$ will be $\tilde{O}\left(T^{(2\alpha+D)/(3\alpha+D)}\right)$. 
Although doubling trick works well in theory, CLUP can suffer from cold-start problems. 
The algorithm we will define in the next section will not require $T$ as an input parameter.
\end{remark}

The regret bound proved in Theorem \ref{theorem:cos} is sublinear in time which guarantees convergence in terms of the average reward, i.e., 
$\lim_{T \rightarrow \infty} \mathrm{E} [R_i(T)]/T = 0$.
For a fixed $\alpha$, the regret becomes linear in the limit as $D$ goes to infinity. On the contrary, when $D$ is fixed, the regret decreases, and in the limit, as $\alpha$ goes to infinity, it becomes $O(T^{2/3})$.
This is intuitive since increasing $D$ means that the dimension of the context increases and therefore the number of hypercubes to explore increases.
While increasing \cem{$\alpha$ means that the level of similarity between any two pairs of contexts increases, i.e., knowing the expected reward of arm $f$ in one context yields more information about its accuracy in another context.}
\rmvt{
In our algorithm we used $d$ as an input parameter and compared with the optimal solution given a fixed $d$. However, the context information can also be adaptively chosen over time. 
For example, in network security, the context can be either time of the day, origin of the data or both. The classifier accuracies will depend on what is used as context information. The regret bound in Theorem \ref{theorem:CLUP} holds even if no data arrives to other classifiers. In general, the regret is much less than this bound. For example, consider the extreme case where the data arrival to each classifier is identical. This implies that classifier $i$ does not need to use the control function $D_2(t)$, since whenever classifier $i$ has sampled all of its own classification functions sufficiently many times, this will also be true for any other classifier $k \in {\cal M}_{-i}$. Therefore, training steps are not required in this scenario. 
}
%
\rmvt{
\subsection{Regret of CLUP for online learning classification functions}

In our analysis we assumed that given a context $x$, the classification function accuracy $\pi_k(x)$ is fixed. This holds when the classification functions are trained a priori, but the learners do not know the accuracy because $k$ is not tested yet. By using our contextual framework, we can also allow the classification functions to learn over time based on the data. \rev{Usually in Big Data applications we cannot have the classifiers being pre-trained as they are often deployed for the first time in a certain setting. For example in \cite{chai2002bayesian}, Bayesian online classifiers are used for text classification and filtering.}
%
We do this by introducing time as a context, thus increasing the context dimension to $d+1$. Time is normalized in interval $[0,1]$ such that $0$ corresponds to $t=0$, 1 corresponds to $t=T$ and each time slot is an interval of length $1/T$. 
For an online learning classifier, intuitively the accuracy is expected to increase with the number of samples, and thus, $\pi_k(x, t)$ will be non-decreasing in time. On the other hand, the increase in the accuracy in a single time step should be bounded. Otherwise, the online learning classifier will be able to classify all possible data streams without any error after a finite number of steps. Because of this, in addition to Assumption  \ref{ass:lipschitz2}, the following should hold for an online learning classifier:
\begin{align*}
\pi_k(x,(t+1)/T) \leq \pi_k(x,t/T) + L T^{-\alpha},
\end{align*}
for some $L$ and $\alpha$. Then we have the following theorem when online learning classifiers are present.
\begin{theorem}\label{theorem:CLUP2}
Let the CLUP algorithm run with exploration control functions $D_1(t) = t^{2\alpha/(3\alpha+d+1)} \log t$, $D_2(t) = F_{\max} t^{2\alpha/(3\alpha+d+1)} \log t$, $D_3(t) = t^{2\alpha/(3\alpha+d+1)} \log t$ and slicing parameter $m_T = T^{1/(3\alpha + d+1)}$. Then,
\add{\vspace{-0.1in}}
\begin{align*}
E[R(T)] &\leq T^{\frac{2\alpha+d+1}{3\alpha+d+1}}
\left( \frac{2 (2 L (d+1)^{\alpha/2}+6)}{\frac{2\alpha+d+1}{3\alpha+d+1}} + 2^{d+1} Z_i \log T \right) \\
&+ T^{\frac{\alpha+d+1}{3\alpha+d+1}} \frac{2^{d+3} (M-1) F_{\max} \beta_2}{\frac{2\alpha}{3\alpha+d+1}} \\
&+ T^{\frac{d}{3\alpha+d+1}} 2^{d+1} (2 Z_i \beta_2 
+ |{\cal K}_i|) + 4 (M-1) F_{\max} \beta_2,
\end{align*}
where $Z_i = {\cal F}_i + (M-1)(F_{\max}+1)$.
\end{theorem} 
\begin{proof}
The proof is the same as proof of Theorem \ref{theorem:CLUP}, with context dimension $d+1$ instaed of $d$.
\end{proof}
The above theorem implies that the regret in the presence of classification functions that learn online based on the data is $O(T^{(2\alpha+d+1)/(3\alpha+d+1)})$. \rev{From the result of Theorem \ref{theorem:CLUP2}, we see that our notion of context can capture any relevant information that can be utilized to improve the classification. Specifically, we showed that by treating time as one dimension of the context we can achieve sublinear regret bounds. Compared to Theorem \ref{theorem:CLUP}, in Theorem \ref{theorem:CLUP2}, the exploration rate is reduced from $O(T^{2\alpha/(2\alpha+d)})$ to $O(T^{2\alpha/(2\alpha+d+1)})$, while the memory requirement is increased from $O(T^{d/(3\alpha+d)})$ to $O(T^{(d+1)/(3\alpha+d+1)})$.}
}

\vspace{-0.1in}
\subsection{Computational Complexity of CLUP}

For each set $p\in {\cal P}_T$, learner $i$ keeps the sample mean of rewards from $F_i+ M_i$ choices, while for a centralized bandit algorithm, the sample mean of the rewards of $|\cup_{j \in {\cal M}} {\cal F}_j|$ arms needs to be kept in memory.
Since the number of sets in ${\cal P}_T$ is upper bounded by $2^D T^{D/(3\alpha+D)} $, the memory requirement is upper bounded by
%
$(\tsp{F_i}  + M_i) 2^D T^{D/(3\alpha+D)}$.
%
This means that the memory requirement is sublinearly increasing in $T$ and thus, in the limit $T \rightarrow \infty$, required memory goes to infinity.
However, CLUP can be modified so that the available memory provides an upper bound on $m_T$. However, in this case the regret bound given in Theorem \ref{theorem:cos} may not hold. 
Also the actual number of hypercubes with at least one context arrival depends on the context arrival process, hence can be very small compared to the worst-case scenario. In that case, it is enough to keep the reward estimates for these hypercubes. 
The following example illustrates that for a practically reasonable time frame, the memory requirement is not very high for a learner compared to a non-contextual centralized implementation (that uses partition $\{ {\cal X} \}$).
For example for $\alpha=1$, $D=1$, we have $2^D T^{D/(3\alpha+D)} = 2 T^{1/4}$.
If learner $i$ learned through $T=10^8$ samples, and if $M=100$, $\tsp{F_j} =100$, for all $j \in {\cal M}$, learner $i$ using CLUP only needs to store at most $40000$ sample mean estimates, while a standard bandit algorithm which does not exploit any context information requires to keep $10000$ sample mean estimates. 
Although, the memory requirement is $4$ times higher than the memory requirement of a standard bandit algorithm, CLUP is suitable for a distributed implementation, learner $i$ does not require any knowledge about the arms of other learners (except an upper bound on the number of arms), and it is shown to converge to the best distributed solution.

\rmvt{
\subsection{CLUP with delayed feedback}

Next, we consider the case when the feedback is delayed. We assume that the label for data instance at time $t$ arrives with an $L_i(t)$ time slot delay, where $L_i(t)$ is a random variable such that $L_i(t) \leq L_{\max}$ with probability one for some $L_{\max}>0$ which is known to the algorithm. Algorithm CLUP is modified so that it keeps in its memory the last $L_{\max}$ labels produced by classification and the indices are updated whenever a true label arrives. We have the following result for delayed label feedback.
\begin{corollary} \label{cor:uniform}
Consider the delayed feedback case where the true label of the data instance at time $t$ arrives at time $t+L_i(t)$, where $L_i(t)$ is a random variable with support in $\{0,1,\ldots, L_{\max}\}$, $L_{\max}>0$ is an integer. Let $R^{\textrm{nd}}(T)$ denote the regret of CLUP with no delay by time $T$, and $R^{\textrm{d}}(T)$ denote the regret of modified CLUP with delay by time $T$. Then we have,
%
$R^{\textrm{d}}(T) \leq L_{\max} + R^{\textrm{nd}}(T)$.
%
\end{corollary}
\remove{
\begin{proof}
By a Chernoff-Hoeffding bound, it can be shown that the probability of deviation of the sample mean accuracy from the true accuracy decays exponentially with the number of samples. A new sample is added to sample mean accuracy whenever the true label of a previous classification arrives. Note that the worst case is when all labels are delayed by $L_{\max}$ time steps. This is equivalent to starting the algorithm with an $L_{\max}$ delay. 
\end{proof}
}

The CLUPt of label delay is additive which does not change the sublinear order of the regret. The memory requirement for CLUP with no delay is $ |{\cal K}_i| (m_T)^d = 2^d (|{\cal F}_i| + M -1) T^{\frac{d}{ 3\alpha + d}}$, while memory requirement for CLUP modified for delay is $L_{\max} + |{\cal K}_i| (m_T)^d$. Therefore, the order of memory CLUPt is also independent of the delay. However, we see that the memory requirement significantly grows with the time horizon $T$ when the dimension of the context space $d$ is high, which makes CLUP infeasible for learning from a very large set of samples. The algorithm we propose in the next section solves this issue by adaptively creating a partition of the context space which requires less than $(m_T)^d$ hypercubes.
}

%% file: zooming_scheme.tex
\section{A distributed adaptive context partitioning algorithm} \label{sec:zooming}

%
%
Intuitively, the loss due to selecting a suboptimal choice for a context can be further minimized if the learners inspect the regions of ${\cal X}$ with large number of context arrivals more carefully, instead of using a uniform partition of ${\cal X}$. 
We do this by introducing the {\em Distributed Context Zooming Algorithm} (DCZA).
\vspace{-0.1in}
\subsection{The DCZA Algorithm}

In the previous section, the partition ${\cal P}_T$ is formed by CLUP at the beginning by choosing the slicing parameter $m_T$.
Differently, DCZA adaptively generates the partition based on how contexts arrive.
Similar to CLUP, using DCZA a learner forms reward estimates for each set in its partition based only on the history related to that set.
Let ${\cal P}_i(t)$ be learner $i$'s partition of ${\cal X}$ at time $t$ and $p_i(t)$ denote the set in ${\cal P}_i(t)$ that contains $x_i(t)$.
Using DCZA, learner $i$ starts with ${\cal P}_i(1) = \{ {\cal X} \}$, then divides ${\cal X}$ into sets with smaller sizes as time goes on and more contexts arrive.
Hence the cardinality of ${\cal P}_i(t)$ increases with $t$.
This division is done in a systematic way to ensure that the tradeoff between the variation
of expected choice rewards inside each set and the number of past observations that are used in reward estimation for each set is balanced.
As a result, the regions of the context space with a lot of context arrivals are covered with sets of smaller sizes than regions of contexts space with few context arrivals. 
In other words, DCZA {\em zooms} into the regions of context space with large number of arrivals.
An illustration that shows partition of CLUP and DCZA is given in Fig. \ref{fig:partition} for $D=1$.
As we discussed in the Section \ref{sec:related} the zooming idea have been used in a variety of multi-armed bandit problems \cite{kleinberg2008multi, bubeck2011x, slivkins2009contextual, dudik2011efficient, langford2007epoch, chu2011contextual}, but there are differences in the problem structure and how zooming is done.

%

The sets in the adaptive partition of each learner are chosen from hypercubes with edge lengths coming from the set $\{ 1, 2^{-1}, 2^{-2}, \ldots \}$.\footnote{Hypercubes have advantages in cooperative contextual bandits because they are disjoint and a learner can pass information to another learner about its partition by only passing the center and edge length of its hypercubes.} 
We call a $D$-dimensional hypercube which has edges of length $2^{-l}$ a level $l$ hypercube (or level $l$ set). For a hypercube $p$, let $l(p)$ denote its level.
%
%
Different from CLUP, the partition of each learner in DCZA can be different since context arrivals to learners can be different. In order to help each other, learners should know about each other's partition. 
For this, whenever a new set of hypercubes is activated by learner $i$, learner $i$ communicates this by sending the center and edge length of one of the hypercubes in the new set of hypercubes to other learners. 
Based on this information, other learners update their partition of learner $i$. 
Thus, at any time slot $t$ all learners know ${\cal P}_i(t)$.
This does not require a learner to keep $M$ different partitions. 
It is enough for each learner to keep ${\cal P}(t) := \bigcup_{i \in {\cal M}} {\cal P}_i(t)$, which is the set of hypercubes that are active for at least one learner at time $t$.
For $p \in {\cal P}(t)$ let $\tau(p)$ be the first time $p$ is activated by one of the learners and for $p \in {\cal P}_i(t)$, let $\tau_i(p)$ be the first time $p$ is activated for learner $i$'s partition.
We will describe the activation process later, after defining the counters of DCZA which are 
initialized and updated differently than CLUP.

$N^i_p(t)$, $p \in {\cal P}_i(t)$ counts the number of context arrivals to set $p$ of learner $i$ (from its own contexts) from times $\{ \tau_i(p), \ldots,  t-1\}$.
For $f \in {\cal F}_i$, $N^i_{f,p}(t)$ counts the number of times arm $f$ is selected in response to contexts arriving to set $p \in {\cal P}(t)$ (from learner $i$'s own contexts or contexts of calling learners) from times $\{ \tau(p), \ldots,  t-1\}$.
Similarly $N^{\textrm{tr},i}_{j,p}(t)$, $p \in {\cal P}_i(t)$ is an estimate on the context arrivals to learner $j$ in set $p$ from all learners except the training phases of learner $j$ and exploration, exploitation phases of learner $i$ from times $\{ \tau(p), \ldots,  t-1\}$.
Finally, $N^i_{j,p}(t)$ counts the number of context arrivals to learner $j$ from exploration and exploitation phases of learner $i$ from times $\{ \tau_i(p), \ldots,  t-1\}$.
Let ${\cal E}^i_{f,p}(t)$, $f \in {\cal F}_i$ be the set of rewards (received or observed) by learner $i$ at times that contribute to the increase of counter $N^i_{f,p}(t)$ and ${\cal E}^i_{j,p}(t)$, $j \in {\cal M}_{-i}$ be the set of rewards received by learner $i$ at times that contribute to the increase of counter $N^i_{j,p}(t)$. We have $\bar{r}^i_{k,p}(t) = (\sum_{r \in {\cal E}^i_{k,p}(t)} r)/|{\cal E}^i_{k,p}(t)|$ for $k \in {\cal K}_i$.
Training, exploration and exploitation within a hypercube $p$ is controlled by control functions
$D_1(p,t) = D_3(p,t) = 2^{2 \alpha l(p)} \log t $ and $D_2(p,t) = F_{\max} 2^{2 \alpha l(p)} \log t$, which depend on the level of hypercube $p$ unlike the control functions $D_1(t)$, $D_2(t)$ and $D_3(t)$ of CLUP, which only depend on the current time. DCZA separates training, exploration and exploitation the same way as CLUP but using control functions $D_1(p,t)$, $D_2(p,t)$, $D_3(p,t)$ instead of $D_1(t)$, $D_2(t)$, $D_3(t)$.

Learner $i$ updates its partition ${\cal P}_i(t)$ as follows. 
At the end of each time slot $t$, learner $i$ checks if $N^i_{p_i(t)}(t+1)$ exceeds a threshold $2^{\rho l(p_i(t))}$, where $\rho$ is the parameter of DCZA that is common to all learners. 
If $N^i_{p_i(t)}(t+1) \geq 2^{\rho l(p_i(t))}$, learner $i$ will divide $p_i(t)$ into $2^D$ level $l(p_i(t))+1$ hypercubes and will note the other learners about its new partition ${\cal P}_i(t+1)$.
With this division $p_i(t)$ is de-activated for learner $i$'s partition.
For a set $p$, let $\tau^{\textrm{fin}}_i(p)$ be the time it is de-activated for learner $i$'s partition.

Similar to CLUP, DCZA also have maximization and cooperation parts.
The maximization part of DCZA is the same as CLUP with training, exploration and exploitation phases.
The only differences are that which phase to enter is determined by comparing the counters defined above with the control functions and in exploitation phase the best choice is selected based on the sample mean estimates defined above.
In the cooperation part at time $t$, learner $i$ explores one of its under-explored arms or chooses its best arm for $p_j(t)$ for learner $j \in {\cal C}_i(t)$ using the counters and sample mean estimates defined above.
Since the operation of DCZA is the same as CLUP except the differences mentioned in this section, we omitted its pseudocode to avoid repetition.

\begin{figure}
\begin{center}
\includegraphics[width=0.7\columnwidth]{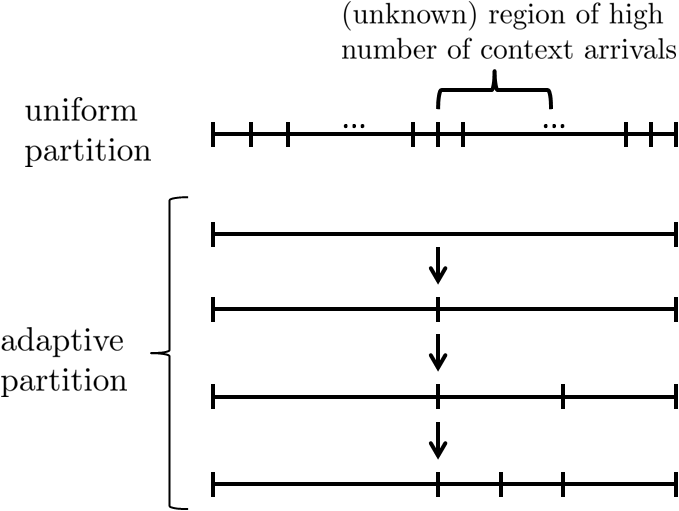}
\vspace{-0.1in}
\caption{An illustration showing how the partition of DCZA differs from the partition of CLUP for $D=1$. As contexts arrive, DCZA {\em zooms} into regions of high number of context arrivals.} 
\vspace{-0.2in}
\label{fig:partition}
\end{center}
\end{figure}

\add{\vspace{-0.15in}}
\subsection{Analysis of the Regret of DCZA}

Our analysis for CLUP in Section \ref{sec:iid} was for worst-case context arrivals. 
This means that the bound in Theorem \ref{theorem:cos} holds even when other learners never call learner $i$ to train it, or other learners never learn by themselves.
In this section we analyze the regret of DCZA under different types of context arrivals. 
Let $K_{i,l}(T)$ be the number of level $l$ hypercubes of learner $i$ that are activated by time $T$. In the following we define two extreme cases of correlation between the contexts arriving to different learners. 
\begin{definition}\label{defn:context}
We call the context arrival process, {\em solo arrivals} if contexts only arrive to learner $i$, {\em identical arrivals} if $x_i(t) = x_j(t)$ for all $i,j \in {\cal M}$, $t =1,\ldots,T$.
%
%
\end{definition}

We start with a simple lemma which gives an upper bound on the highest level hypercube that is active at any time $t$.
\begin{lemma}\label{lemma:levelbound}
All the active hypercubes $p \in {\cal P}(t)$ at time $t$ have at most a level of 
%
$\rho^{-1} \log_2 t + 1$.
%
\end{lemma}
\remove{
\begin{proof}
Let $l'+1$ be the level of the highest level active hypercube. We must have
%
$\sum_{l=0}^{l'} 2^{\rho l} < t$,
otherwise the highest level active hypercube's level will be less than $l'+1$. We have,
%
$(2^{\rho(l'+1)}-1)/(2^\rho-1) < t
\Rightarrow 2^{\rho l'} < \tsp{t} 
\Rightarrow l' < \rho^{-1} \log_2 t$.
%
\end{proof}
}

In order to analyze the regret of DCZA, we first bound the regret due to trainings and explorations in a level $l$ hypercube. We do this for the solo and identical context arrival cases separately.

\begin{lemma} \label{lemma:adapexplore}
Consider all learners that run DCZA with parameters 
$D_1(p,t) = D_3(p,t) =  2^{2 \alpha l(p)} \log t $ and $D_2(p,t) = F_{\max} 2^{2 \alpha l(p)}  \log t $. Then, for any level $l$ hypercube the regret of learner $i$ due to trainings and explorations by time $T$ is bounded above by
(i) $ 2 \tsp{Z_i} (2^{2 \alpha l } \log T +1)$ for solo context arrivals,
(ii) $2 \tsp{K_i} (2^{2 \alpha l } \log T +1)$ for identical context arrivals (given $F_i \geq F_j$, $j \in {\cal M}_{-i}$).\footnote{In order for the bound for identical context arrivals to hold for learner $i$ we require that $F_i \geq F_j$, $j \in {\cal M}_{-i}$. Hence, in order for the bound for identical context arrivals to hold for all learners, we require $F_i = F_j$ for all $i,j \in {\cal M}$.}
\end{lemma}
\remove{
\begin{proof}
The proof is similar to Lemma \ref{lemma:explorations}.
Note that when the context arriving to each learner is the same and $|{\cal F}_i| \geq |{\cal F}_{j}|$, $j \in {\cal M}_{-i}$, we have $N^{i,tr}_{j,p}(t) > D_2(p,t)$ for all $j \in {\cal M}_{-i}$ whenever $N^i_{f,p}(t) > D_1(p,t)$ for all $f \in {\cal F}_i$.
\end{proof}
}

We define the set of suboptimal choices and arms for learner $i$ in DCZA a little differently than CLUP (suboptimality depends on the level of the hypercube but not on time), using the same notation as in the analysis of CLUP.
Let
\begin{align}
{\cal L}^i_{p} &:= \left\{ k \in {\cal K}_i :  \underline{\mu}^i_{k^*_i(p),p} - \overline{\mu}^i_{k,p} > A^* L D^{\alpha/2} 2^{- l(p) \alpha} \right\} \label{eqn:DCZAsubchoice}
\end{align}
be the set of suboptimal choices of learner $i$ for a hypercube $p$, and
\begin{align}
\hspace{-0.1in} {\cal F}^j_{p} &:= \left\{ f \in {\cal F}_j :  \underline{\pi}_{f^*_j(p),p} - \overline{\pi}_{f,p} > A^* L D^{\alpha/2} 2^{- l(p) \alpha} \right\} \label{eqn:DCZAsubarm}
\end{align}
be the set of suboptimal arms of learner $j$ for hypercube $p$, where 
$A^* = 2 + 4/(LD^{\alpha/2})$.

In the next lemma we bound the regret due to choosing suboptimal choices in the exploitation steps of learner $i$. 

\begin{lemma} \label{lemma:suboptimal}
Consider all learners running DCZA with parameters $\rho>0$, $D_1(p,t) = D_3(p,t) = 2^{2 \alpha l(p)} \log t$ and $D_2(p,t) = F_{\max} 2^{2 \alpha l(p)} \log t$.
Then, we have 
\begin{align}
\mathrm{E} [R^s_i(T)] 
&\leq 4 (M_i + F_i) \beta_2  \notag \\
&+ 4 (M_i + F_i) M_i F_{\max} \beta_2 \sum_{t=1}^T 2^{- \alpha l(p_i(t))} .      \notag
\end{align}
\end{lemma}
\begin{proof}
The proof of this lemma is similar to the proof of Lemma \ref{lemma:suboptimal}, thus some steps are omitted.
${\cal W}^i(t)$ and ${\cal V}^i_{k}(t)$ are defined the same way as in Lemma \ref{lemma:suboptimal}.
${\cal B}^i_{j,p_i(t)}(t)$ denotes the event that at most $2^{\alpha l(p_i(t))}$ samples in ${\cal E}^i_{j,p_i(t)}(t)$ are collected from the suboptimal arms of learner $j$ in ${\cal F}^j_{p_i(t)}$, and ${\cal B}^i(t) := \bigcap_{j \in {\cal M}_{-i}} {\cal B}^i_{j,p_i(t)}(t)$.
We have
$\mathrm{E} [R^s_{i}(T)] \leq \tsp{2} \sum_{t=1}^T \sum_{k \in {\cal L}^i_{p_i(t)}} 
\mathrm{P} ({\cal V}^i_{k}(t), {\cal W}^i(t))$.

Similar to Lemma \ref{lemma:suboptimal}, we have
\begin{align}
& \mathrm{P} \left( {\cal V}^i_{k}(t), {\cal W}^i(t) \right) \notag \\
& \leq \mathrm{P} \left( \hat{\mu}^i_{k,p_i(t)}(t) \geq \overline{\mu}^i_{k,p_i(t)} + H_{t}, {\cal W}^i(t), {\cal B}^i(t)  \right) \notag  \\
&+ \mathrm{P} \left( \hat{\mu}^i_{k^*_i,p_i(t)}(t) \leq \underline{\mu}^i_{k^*_i,p_i(t)} - H_t, {\cal W}^i(t), {\cal B}^i(t) \right) \notag \\
&+ \mathrm{P} \left( \hat{\mu}^i_{k,p_i(t)}(t) \geq \hat{\mu}^i_{k^*_i,p_i(t)}(t), 
\hat{\mu}^i_{k,p_i(t)}(t) < \overline{\mu}^i_{k,p_i(t)} + H_t, \right. \notag \\
& \left. \hat{\mu}^i_{k^*_i,p_i(t)}(t) > \underline{\mu}^i_{k^*_i,p_i(t)} - H_t,
{\cal W}^i(t), {\cal B}^i(t)  \right) \notag \\
&+ \mathrm{P} ({\cal B}^i(t)^c, {\cal W}^i(t) ). \notag 
\end{align}
Letting
\begin{align}
H_t =  (LD^{\alpha/2} +2) 2^{-\alpha l(p_i(t))}      \notag
\end{align}
we have
\begin{align}
\mathrm{P} \left( \hat{\mu}^i_{k,p_i(t)}(t) \geq \overline{\mu}^i_{k,p_i(t)} + H_t, {\cal W}^i(t), {\cal B}^i(t)  \right) 
& \leq t^{-2}  \notag \\
\mathrm{P} \left( \hat{\mu}^i_{k^*_i,p_i(t)}(t) \leq \underline{\mu}^i_{k^*_i,p_i(t)} - H_t, {\cal W}^i(t), {\cal B}^i(t) \right) 
& \leq t^{-2} . \notag 
\end{align}
Since $2 H_t \leq A^* L D^{\alpha/2} 2^{- l(p_i(t)) \alpha}$,
\begin{align}
&\mathrm{P} \left( \hat{\mu}^i_{k,p_i(t)}(t) \geq \hat{\mu}^i_{k^*_i,p_i(t)}(t), 
\hat{\mu}^i_{k,p_i(t)}(t) < \overline{\mu}^i_{k,p_i(t)} + H_t, \right. \notag \\
& \left. \hat{\mu}^i_{k^*_i,p_i(t)}(t) > \underline{\mu}^i_{k^*_i,p_i(t)} - H_t,
{\cal W}^i(t), {\cal B}^i(t)  \right)  = 0 .      \notag
\end{align}

Similar to the proof of Lemma \ref{lemma:suboptimal}, we have
\begin{align}
\mathrm{P} (\Xi^i_{j,p_i(t)}) &\leq 2 F_j t^{-2}      \notag \\
\mathrm{E} [X^i_{j,p_i(t)}]  & \leq 2 F_j \beta_2 \notag \\
\mathrm{P} ( {\cal B}^i_{j,p_i(t)}(t)^c, {\cal W}^i(t) ) &\leq 2 F_j \beta_2 2^{-\alpha l(p_i(t))} \notag \\
\mathrm{P} ( {\cal B}^i(t)^c, {\cal W}^i(t) )  & \leq 2 M_i F_{\max} \beta_2 2^{-\alpha l(p_i(t))} . \notag
\end{align}
Hence,
\begin{align}
\mathrm{P} \left( {\cal V}^i_{k}(t), {\cal W}^i(t) \right) \leq 2 t^{-2} + 2 M_i F_{\max} \beta_2 2^{-\alpha l(p_i(t))} .      \notag
\end{align}
\end{proof}

In the next lemma we bound the regret of learner $i$ due to selecting near optimal choices.
\begin{lemma}\label{lemma:adapnearopt}
Consider all learners running DCZA with parameters $\rho>0$, $D_1(p,t) = D_3(p,t) = 2^{2 \alpha l(p)} \log t$ and $D_2(p,t) = F_{\max} 2^{2 \alpha l(p)}\log t$. Then, we have
\begin{align*}
\mathrm{E}[R^n_i(T)] \leq  4 M_i F_{\max} \beta_2
+ 2 (3+A^*) L D^{\alpha/2} \sum_{t=1}^T 2^{-\alpha l(p_i(t))} . 
\end{align*}
\end{lemma}
\begin{proof}
For any $k \in  {\cal K}_i -  {\cal L}^i_{p_i(t)}$ and $x \in p_i(t)$, we have 
$\mu^i_{ k^*_i(x) }(x) - \mu^i_k(x) \leq (3+A^*) L D^{\alpha/2} 2^{-l(p_i(t)) \alpha}$. Similarly for any $j \in {\cal M}$, $f \in {\cal F}_j -  {\cal F}^j_{p_i(t)}(t)$ and $x \in p_i(t)$, we have 
$\pi_{ f^*_j(x) }(x) - \pi_f(x) \leq (3+A^*) L D^{\alpha/2} 2^{-l(p_i(t)) \alpha}$.

As in the proof of Lemma \ref{lemma:suboptimal}, we have
%
$\mathrm{P} (\Xi^i_{j,p_i(t)}(t)) \leq 2 F_{\max} t^{-2}$.
Thus, when a near optimal learner $j \in {\cal M}_{-i} \cap ({\cal K}_i -  {\cal L}^i_p)$ is called by learner $i$ at time $t$, the contribution to the regret from suboptimal arms of $j$ is bounded by $4 F_{\max} t^{-2}$.
The one-slot regret of any near optimal arm of any near optimal learner \tsp{$j \in  {\cal M}_{-i} \cap ({\cal K}_i -  {\cal L}^i_p)$} is bounded by $2 (3+A^*) L D^{\alpha/2} 2^{-l(p) \alpha}$.
The one-step regret of any near optimal arm \tsp{$f \in {\cal F}_{i} \cap ({\cal K}_i -  {\cal L}^i_p)$} is bounded by $\tsp{(3+A^*)} L D^{\alpha/2} 2^{-l(p) \alpha}$. The result is obtained by taking the sum up to time $T$. 
\end{proof}

Next, we combine the results from Lemmas \ref{lemma:adapexplore}, \ref{lemma:suboptimal} and \ref{lemma:adapnearopt} to obtain regret bounds as a function of the number of hypercubes of each level that are activated up to time $T$. 

\begin{theorem}\label{thm:adaptivemain}
\ntsp{Consider all learners running DCZA with parameters $\rho>0$, $D_1(p,t) = D_3(p,t) = 2^{2 \alpha l(p)} \log t$ and $D_2(p,t) = F_{\max} 2^{2 \alpha l(p)} \log t$. Then, for {\em solo arrivals}, we have 
\begin{align}
R_i(T)  & \leq  2 C_1 \sum_{l=0}^{(\log_2 T/\rho) + 1} K_{i,l}(T) 2^{2 \alpha l} \log T \notag \\
&+ C_2  \sum_{l=0}^{(\log_2 T/\rho) + 1} K_{i,l}(T) 2^{ (\rho-\alpha) l}  \notag \\
&+ 2 C_1 \sum_{l=0}^{(\log_2 T/\rho) + 1} K_{i,l}(T) + C_0 \notag 
\end{align}
where 
$C_0 = 4 \beta_2 (M_i + F_i + M_i F_{\max})$,  
$C_1 = Z_i$ for {\em solo arrivals} and $C_1 = K_i$ for {\em identical arrivals} and
$C_2 = 4 (M_i + F_i) M_i F_{\max} \beta_2 + 2 (3 + A^*) L D^{\alpha/2}$.}
\end{theorem}
\begin{proof}
The result follows from summing the results of Lemmas 6, 7 and 8 and using Lemma 5. 
\end{proof}

Although the result in Theorem \ref{thm:adaptivemain} bounds the regret of DCZA for an arbitrary context arrival process in terms of $K_{i,l}(T)$'s, it is possible to obtain context arrival process independent regret bounds by considering the worst-case context arrivals. The next corollary shows that the worst-case regret bound of DCZA matches with the worst-case regret bound of CLUP derived in Theorem 1.

\begin{corollary}\label{corr:DCZAworstcase}
\ntsp{Consider all learners running DCZA with parameters $\rho= 3\alpha$, $D_1(p,t) = D_3(p,t) = 2^{2 \alpha l(p)} \log t$ and $D_2(p,t) = F_{\max} 2^{2 \alpha l(p)} \log t$. Then, the worst-case regret of learner $i$ is bounded by 
\begin{align}
R_i(T) & \leq 2^{2 (D+2\alpha) } (2 C_1 \log T + C_2 ) T^{\frac{2\alpha+D}{3\alpha+D}}     \notag \\
&+ 2 C_1 2^{2D} T^{\frac{D}{3\alpha+D}} 
+ C_0 \notag
\end{align}
where $C_0$, $C_1$ and $C_2$ are given in Theorem 2.}
\end{corollary}
\begin{proof}
Since hypercube $p$ remains active for at most $2^{\rho l(p)}$ context arrivals within that hypercube, combining the results of Lemmas 7 and 8, the expected loss in hypercube $p$ in exploitation slots is  at most $C_2 2^{(\rho-\alpha) l(p)}$, where $C_2$ is defined in Theorem 2. However, the expected loss in hypercube $p$ due to trainings and explorations is at least $C 2^{2 \alpha l(p)}$ for some constant $C>0$, and is at most $2 Z_i (2^{2\alpha l(p)} \log T +1)$ as given in Lemma 6. In order to balance the regret due to trainings and explorations with the regret incurred in exploitation within $p$ we set $\rho=3\alpha$. 

In the worst-case context arrivals, contexts arrive in a way that all level $l$ hypercubes are divided into level $l+1$ hypercubes before contexts start arriving to any of the level $l+1$ hypercubes. In this way, the number of hypercubes to train and explore is maximized. Let $l_{\max}$ be the hypercube with the maximum level that had at least one context arrival on or before $T$ in the worst-case context arrivals. We must have
\begin{align}
\sum_{l=0}^{l_{\max}-1} 2^{D l} 2^{3 \alpha l} < T .    \notag
\end{align}
Otherwise, no hypercube with level $l_{\max}$ will have a context arrival by time $T$. From the above equation we get $l_{\max} < 1 + (\log_2 T) / (D+ 3 \alpha)$. 
Thus, 
\begin{align}
R_i(T)  & \leq 2 C_1 \sum_{l=0}^{l_{\max}} 2^{D l} 2^{2 \alpha l} \log T   
+ C_2 \sum_{l=0}^{l_{\max}} 2^{D l} 2^{2 \alpha l} \notag \\
&+ 2 C_1 \sum_{l=0}^{l_{\max}}  2^{D l} + C_0 . \notag 
\end{align}
\end{proof}

%% file: discussion.tex
\vspace{-0.1in}
\section{Discussion} \label{sec:discuss}

\subsection{Necessity of the Training Phase} \label{sec:necessity}

In this subsection, we prove that the training phase is necessary to achieve sublinear regret for the cooperative contextual bandit problem for algorithms of the type CLUP and DCZA (without the training phase) which use (i) exploration control functions of the form $C t^z \log t$, for constants $C>0$, $z>0$; (ii) form a finite partition of the context space; and (iii) use the sample mean estimator within each hypercube in the partition. We call this class of algorithms {\em Simple Separation of Exploration and Exploitation} (SSEE) algorithms.
In order to show this, we consider a special case of expected arm rewards and context arrivals and show that independent of the rate of explorations, the regret of an SSEE algorithm is linear in time for any exploration control function $D_i(t)$\footnote{Here $D_i(t)$ is the control function that controls when to explore or exploit the choices in ${\cal K}_i$ for learner $i$.} of the form $C t^z \log t$  for learner $i$ (exploration functions of learners can be different).
Although, our proof does not consider index-based learning algorithms, we think that similar to our construction in Theorem \ref{thm:necessitytrain}, problem instances which will give linear regret can be constructed for any type of index policy without the training phase. 

\begin{theorem}\label{thm:necessitytrain}
Without the training phase, the regret of any SSEE algorithm is linear in time.
\end{theorem}
\begin{proof}
We will construct a problem instance for which the statement of the theorem is valid. Assume that all costs $d^i_k$, $k \in {\cal K}_i$, $i \in {\cal M}$ are zero.
Let $M=2$. 
Consider a hypercube $p$.
We assume that at all time slots context $x^* \in p$ arrives to learner $1$, and all the contexts that are arriving to learner 2 are outside $p$.
Learner 1 has only a single arm $m$, learner 2 has two arms $b$ and $g$.
With an abuse of notation, we denote the expected reward of an arm $f \in \{m,b,g\}$ at context $x^*$ as $\pi_f$.
Assume that the arm rewards are drawn from $\{0,1\}$ and the following is true for expected arm rewards:
\begin{align}
\pi_b + C_K \delta <\pi_m < \pi_g - \delta < \pi_m + \delta \label{eqn:specform}
\end{align}
for some $\delta>0$, $C_K>0$,  where the value of $C_K$ will be specified later.
Assume that learner 1's exploration control function is $D_1(t) = t^z \log t$, and
learner 2's exploration control function is $D_2(t) = t^z \log t/K$ for some $K \geq 1$, $0<z<1$.\footnote{Given two control functions of the form $C_i t^z \log t$, $i \in \{1,2\}$, we can always normalize them such that one of them is $t^z \log t$ and the other one is $t^z \log t/K$, and then construct the problem instance that gives linear regret based on the normalized control functions.}

When we have $K=1$, when called by learner 1 in its explorations, learner 2 may always choose its suboptimal arm $b$ since it is under-explored for learner 2.
If this happens, then in exploitations learner 1 will almost always choose its own arm instead of learner 2, because it had estimated the accuracy of learner 2 for $x^*$ incorrectly because the random rewards in explorations of learner 2 came from $b$.
By letting $K \geq 1$, we also consider cases where only a fraction of reward samples of learner 2 for learner 1 comes from the suboptimal arm $b$. We will show that for any value of $K \geq 1$, there exists a problem instance of the form given in (\ref{eqn:specform})
such that learner 1's regret is linear in time.
Let $E_t$ be the event that time $t$ is an exploitation slot for learner $1$.
Let $\hat{\pi}_m(t), \hat{\pi}_2(t) $ be the sample mean reward of arm $m$ and learner $2$ for learner $1$ at time $t$ respectively.
Let $\xi_{\tau}$ be the event that learner 1 exploits for the $\tau$th time by choosing its own arm. Denote the time of the $\tau$th exploitation of learner 1 by $\tau(t)$.
We will show that for any finite $\tau$, $\mathrm{P} (\xi_\tau, \ldots, \xi_1) \geq 1/2$. 
We have by the chain rule
\begin{align}
\mathrm{P} (\xi_\tau, \ldots, \xi_1) &= \mathrm{P} (\xi_\tau | \xi_{\tau-1}, \ldots, \xi_1) \textrm{Pr}(\xi_{\tau-1}| \xi_{\tau-2}, \ldots, \xi_1) \notag \\
 &\ldots \mathrm{P} (\xi_1). \label{eqn:boundedtau}
\end{align}

We will continue by bounding $\mathrm{P} (\xi_\tau | \xi_{\tau-1}, \ldots, \xi_1)$.
When the event $E_{\tau(t)} \cap \xi_{\tau-1} \cap \ldots \cap \xi_1$ happens, we know that at least $\lceil \tau(t)^z \log \tau(t)/K \rceil$ of $\lceil \tau(t)^z \log \tau(t) \rceil$ reward samples of learner 2 for learner 1 comes from $b$.
Let
$A_t := \{ \hat{\pi}_m(t) > \pi_m - \epsilon_1 \}$,
$B_t := \{ \hat{\pi}_2(t) < \pi_g - \epsilon_2 \}$
and $C_t := \{ \hat{\pi}_2(t) <  \hat{\pi}_m(t)   \}$, for $\epsilon_1>0, \epsilon_2>0$.
Given $\epsilon_2 \geq \epsilon_1 + 2 \delta$, we have $(A_t \cap B_t) \subset C_t$.
Consider the event $\{A_t^C, E_t \}$. Since on $E_t$, learner 1 selected $m$ at least $t^z \log t$ times (given that $z$ is large enough such that the reward estimate of learner 1's own arm is accurate), we have $\mathrm{P}(A_t^C, E_t) \leq 1/(2 t^2)$, using a Chernoff bound. Let $N_g(t)$ ($N_b(t)$) be the number of times learner $2$ has chosen arm $g$ ($b$) when called by learner 1 by time $t$. Let $r_g(t)$ ($r_b(t)$) be the random reward of arm $g$ ($b$) when it is chosen for the $t$th time by learner 2.
For $\eta_1>0$, $\eta_2>0$, let $Z_1(t) := \{ (\sum_{t'=1}^{N_g(t)} r_g(t')) /N_g(t) < \pi_g + \eta_1 \}$ and $Z_2(t) := \{ (\sum_{t'=1}^{N_b(t) } r_b(t')) /N_b(t) < \pi_b + \eta_2 \}$.
On the event $E_{\tau(t)} \cap \xi_{\tau-1} \cap \ldots \cap \xi_1$, we have $N_g(\tau(t)) / N_b(\tau(t)) \leq K$.
Since $\hat{\pi}_2(t) = \left( \sum_{t'=1}^{N_b(t) } r_b(t') + \sum_{t'=1}^{N_g(t)} r_g(t')  \right)/(N_b(t)  + N_g(t) )$,
We have
\begin{align}
& Z_1(t) \cap Z_2(t) \notag \\
&\hspace{-0.2in}  \Rightarrow \hat{\pi}_2(t) < \frac{N_g(t) \pi_g + N_b(t) \pi_b + \eta_1 N_g(t) + \eta_2 N_b(t)}{N_b(t) + N_g(t)}. \label{eqn:bound}
\end{align}
If
\begin{align}
& \pi_g - \pi_b > \frac{N_g(t) }{N_b(t)} (\eta_1 + \epsilon_2) + (\eta_2 + \epsilon_2) \label{eqn:ngt}
\end{align}
then, it can be shown that the right hand side of (\ref{eqn:bound}) is less than $\pi_g - \epsilon_2$.
Thus given that (\ref{eqn:ngt}) holds, we have $Z_1(t) \cap Z_2(t) \subset B_t$.
But on the event $E_{\tau(t)} \cap \xi_{\tau-1} \cap \ldots \cap \xi_1$, (\ref{eqn:ngt}) holds at $\tau(t)$ when
$\pi_g - \pi_b > K (\eta_1 + \epsilon_2) + (\eta_2 + \epsilon_2)$.    
Note that if we take $\epsilon_1 = \eta_1 = \eta_2 = \delta/2$, and $\epsilon_2 = \epsilon_1 + 2 \delta = 5\delta/2$ the statement above holds for a problem instance with $C_K > 3K +3$. 
Since at any exploitation slot $t$, at least $\lceil t^z \log t /K \rceil$ samples are taken by learner 2 from both arms $b$ and $g$, we have $\mathrm{P} (Z_1(\tau(t))^C) \leq 1/(4\tau(t)^2)$ and $\mathrm{P} (Z_2(\tau(t))^C) \leq 1/(4\tau(t)^2)$ by a Chernoff bound (again for $z$ large enough as in the proofs of Theorems 1 and 2).
Thus $\mathrm{P} (B_{\tau(t)})^C \leq \mathrm{P} (Z_1(\tau(t))^C) + \mathrm{P} (Z_2(\tau(t))^C) \leq 1/(2\tau(t)^2)$.
Hence $\mathrm{P} (C_{\tau(t)}^C) \leq \mathrm{P} (A_{\tau(t)}^C)  + \mathrm{P} (B_{\tau(t)}^C) \leq  1/(\tau(t)^2)$, and $\mathrm{P} (C_{\tau(t)}) > 1-1/(\tau(t)^2)$.
Continuing from (\ref{eqn:boundedtau}), we have 
\begin{align}
 \mathrm{P} (\xi_\tau, \ldots, \xi_1) &= \left( 1 - 1/(\tau(t)^2)\right)   \left( 1 - 1/((\tau-1)(t)^2)\right) \notag \\
 &  \ldots 
 \left( 1 - 1/((1)(t)^2)\right)     \notag \\
 & \geq \Pi_{t'=2}^{\tau(t)} \left( 1 - 1/(t')^2\right) > 1/2 
\end{align}
for all $\tau$.
This result implies that with probability greater than one half, learner 1 chooses its own arm at all of its exploitation slots, resulting in an expected per-slot regret of $\pi_g - \pi_m >\delta$. Hence the regret is linear in time.
\end{proof}

\vspace{-0.1in}
\subsection{Comparison of CLUP and DCZA}

In this subsection we assess the computation and memory requirements of DCZA and compare it with CLUP. DCZA needs to keep the sample mean reward estimates of \tsp{$K_i$} choices for each active hypercube. A level $l$ active hypercube becomes inactive if the context arrivals to that hypercube exceeds $2^{\rho l}$.
Because of this, the number of active hypercubes at any time $T$ may be much smaller than the number of activated hypercubes by time $T$.
In the best-case, only one level $l$ hypercube experiences context arrivals, then when that hypercube is divided into level $l+1$ hypercubes, only one of these hypercubes experiences context arrivals and so on. In this case, DCZA run with $\rho=3\alpha$ creates at most $1 +  (\log_2 T)/(3\alpha)$ hypercubes (using Lemma 5). In the worst-case (given in Corollary 1), DCZA creates at most 
$2^{2D} T^{D/(3\alpha + D)}$ hypercubes. 
Recall that for any $D$ and $\alpha$, the number of hypercubes of CLUP creates is $O( T^{D/(3\alpha+D)})$. Hence, in practice the memory requirement of DCZA can be much smaller than CLUP which requires to keep the estimates for every hypercube at all times. 
Finally DCZA does not require final time $T$ as in input while CLUP requires it.
Although CLUP can be combined with the doubling trick to make it independent of $T$, this makes the constants that multiply the time order of the regret large.

\comment{
\begin{table}[t]
\centering
{\fontsize{8}{8}\selectfont
\begin{tabular}{|l|c|c|c|}
\hline
Case & CLUP & DCZA \\
\hline
\textbf{C1} & $\tilde{O} \left(M F_{\max} T^{\frac{2\alpha+D}{3\alpha+D}} \right)$ & $\tilde{O} \left(M F_{\max} T^{f_1(\alpha,D)}\right)$ \\
\hline
\textbf{C2} & $\tilde{O} \left( K_i T^{\frac{2\alpha+D}{3\alpha+D}} \right)$ & $\tilde{O}\left( K_i T^{f_1(\alpha,D)}\right)$ \\
\hline
\textbf{C3} & $\tilde{O} \left(M F_{\max}  T^{\frac{2\alpha}{3\alpha+D}} \right)$ & $\tilde{O}\left(M F_{\max} T^{2/3} \right)$ \\
\hline
\textbf{C4} & $\tilde{O} \left( K_i T^{\frac{2\alpha}{3\alpha+D}} \right)$ &  $\tilde{O}\left( K_i T^{2/3} \right)$ \\
\hline
\end{tabular}
}
\caption{Comparison of Regrets of CLUP and DCZA}
\vspace{-0.3in}
\label{tab:compregret}
\end{table}
}

\add{
\begin{table}[t]
\centering
{\fontsize{7}{6}\selectfont
\setlength{\tabcolsep}{1em}
\begin{tabular}{|l|c|c|}
\hline
& CLUP & DDZA \\
\hline
worst arrival & $O \left(M F_{\max} T^{\frac{2\alpha+d}{3\alpha+d}} \right)$ & $O \left(M F_{\max} T^{f_1(\alpha,d)}\right)$  \\
and correlation& & \\
\hline
worst arrival, & $O \left(|{\cal K}_i| T^{\frac{2\alpha+d}{3\alpha+d}} \right)$ & $O\left( K_i T^{f_1(\alpha,d)}\right)$  \\
best correlation & & \\
\hline
best arrival,& $O \left(M F_{\max}  T^{\frac{2\alpha}{3\alpha+d}} \right)$ & $O\left(M F_{\max} T^{2/3} \right)$ \\
worst correlation&& \\
\hline
best arrival & $O \left(|{\cal K}_i| T^{\frac{2\alpha}{3\alpha+d}} \right)$ & $O\left(|{\cal K}_i| T^{2/3} \right)$ \\
and correlation & & \\
\hline
\end{tabular}
}
\caption{Comparison CLUP and DDZA}
\label{tab:compregret}
\vspace{-0.4in}
\end{table}
}
\rmv{
\subsection{Cooperation among the learners}
In our analysis we assumed that learner $i$ can call any other learner $k \in {\cal M}_{-i}$ with a CLUPt $d_k$, and $k \in {\cal M}_{-i}$ always sends back its prediction in the same time slot. However, learner $k$ also has to classify its own data stream thus it may not be possible for it to classify $i$s data without delay. We considered the effect of a bounded delay in Corollary \ref{cor:uniform}. 
We note that there is also a CLUPt for learner $k$ associated with communicating with learner $i$, but it is small since learner $k$ only needs to send $i$ its prediction but not the data. However, even tough learner $k$ does not have immediate benefit from classifying $i$'s data, similar to $i$ it can use other learners to increase its prediction accuracy minus classification CLUPt. In this paper we assumed that the learners are cooperative and there are no strategic interactions. Each learner should classify another learner's data when requested. This does not affect the optimal policy we derive for learner $i$ since the requests of other learner's is independent of actions of $i$.
}
%
\comment{
\vspace{-0.1in}
\subsection{Context dependent CLUPts}

In our analysis we assumed that classification CLUPt $d_k$ is constant and does not depend on the context. In general this CLUPt can be context or data dependent. 
For example, consider a network and two arms $l$ and $k$ for which $d_k >> d_l$ but $0<\pi_k(x) - \pi_l(x)<<1$. 
The network can go down when attacked at a specific context $x'$, thus, the expected loss $g_l(x')$ for arm $l$ can be much higher than the expected loss $g_k(x')$ for arm $k$. Then, in the optimal solution, arm $k$ will be chosen instead of arm $l$ even though $d_k >> d_l$.
Our results will hold for any general reward function $g_k(x)$ for which Assumption \ref{ass:lipschitz2} holds. 
}

%% file: app_seriesbound.tex
For $\rho>0$, $\rho \neq 1$,
%
$\sum_{t=1}^{T} 1/(t^\rho) \leq 1 + (T^{1-\rho} -1)/(1-\rho)$.
%
\begin{proof}
See \cite{chlebus2009approximate}.
\end{proof}

%% file: notation.tex
\textbf{Mathematical operators}
\begin{itemize}
\item $O(\cdot)$: Big O notation.
\item $\tilde{O}(\cdot)$: Big O notation with logarithmic terms hidden.
\item $\mathrm{I}(A)$: indicator function of event $A$. 
\item $A^c$ or $A^C$: complement of set $A$.
\end{itemize}

\textbf{Notation related to underlying system}
\begin{itemize}
\item ${\cal M}$: Set of learners. $M=|{\cal M}|$.
\item ${\cal F}_i$: Set of arms of learner $i$. $F_i=|{\cal F}_i|$.
\item ${\cal M}_{-i}$: Set of learners except $i$. $M_i=|{\cal M}_{-i}|$.
\item ${\cal K}_i$: Set of choices of learner $i$. $K_i=|{\cal K}_i|$.
\item ${\cal F}$: Set of all arms.
\item ${\cal X} = [0,1]^D$: Context space.
\item $D$: Dimension of the context space.
\item $\pi_f(x)$: Expected reward of arm $f \in {\cal F}$ for context $x$.
\item $\pi_j(x)$: Expected reward of learner $j$'s best arm for context $x$.
\item $d^i_k$: Cost of selecting choice $k \in {\cal K}_i$ for learner $i$.
\item $\mu^i_k(x) = \pi_k(x) - d^i_k$: Expected net reward of learner $i$ from choice $k$ for context $x$.
\item $k^*_i(x)$: Best choice (highest expected net reward) for learner $i$ for context $x$.
\item $f^*_i(x)$: Best arm (highest expected reward) of learner $j$ for context $x$.
\item $L$: H\"{o}lder constant. $\alpha$: H\"{o}lder exponent.
\end{itemize}

\textbf{Notation related to algorithms}
\begin{itemize}
\item $D_1(t), D_2(t), D_3(t)$: Control functions.
\item $p$: Index for set of contexts (hypercube).
\item $m_T$: Number of slices for each dimension of the context for CLUP.
\item ${\cal P}_T$: Partition of ${\cal X}$ for CLUP.
\item ${\cal P}_i(t)$: Learner $i$'s adaptive partition of ${\cal X}$ at time $t$ for DCZA.
\item ${\cal P}(t)$: Union of partitions of ${\cal X}$ of all learners for DCZA.
\item $p_i(t)$: The set in ${\cal P}_i(t)$ that contains $x_i(t)$.
\item ${\cal M}^{\textrm{uc}}_{i,p}(t)$: Set of learners who are training candidates of learner $i$ at time $t$ for set $p$ of learner $i$'s partition.
\item ${\cal M}^{\textrm{ut}}_{i,p}(t)$: Set of learners who are under-trained by learner $i$ at time $t$ for set $p$ of learner $i$'s partition.
\item ${\cal M}^{\textrm{ue}}_{i,p}(t)$: Set of learners who are under-explored by learner $i$ at time $t$ for set $p$ of learner $i$'s partition.
\item ${\cal M}^{\textrm{uc}}_{i,p}(t)$: Set of learners who are training candidates of learner $i$ at time $t$ for set $p$ of learner $i$'s partition.
\end{itemize}